\documentclass{article}

\PassOptionsToPackage{numbers, compress}{natbib}

\usepackage[final]{neurips_glfrontiers_2022}
\usepackage{wrapfig}

\usepackage[utf8]{inputenc} 
\usepackage[T1]{fontenc}    
\usepackage{hyperref}       
\usepackage{url}            
\usepackage{booktabs}       
\usepackage{amsfonts}       
\usepackage{nicefrac}       
\usepackage{microtype}      
\usepackage{xcolor}         

\usepackage{amsfonts,amsmath,amssymb,amsthm,bm}
 \usepackage{mathtools}
\usepackage{graphicx,psfrag,epsf}
\usepackage{enumerate}

\usepackage{dsfont}
\usepackage{comment}
\usepackage{algorithm}
\usepackage{algpseudocode}
\algnewcommand{\Inputs}[1]{%
  \State \textbf{Inputs:}
  \Statex \hspace*{\algorithmicindent}\parbox[t]{.8\linewidth}{\raggedright #1}
}
\algnewcommand{\Initialize}[1]{%
  \State \textbf{Initialize:}
  \Statex \hspace*{\algorithmicindent}\parbox[t]{.8\linewidth}{\raggedright #1}
}

\newcommand{\R}{\mathbb{R}}	

\newcommand{\E}{\mathbb{E}}

\newtheorem{thm}{Theorem}[section]

\newtheorem{prop}[thm]{Proposition}

\theoremstyle{definition}
\newtheorem{defn}{Definition}[section]

\title{From Local to Global: Spectral-Inspired Graph Neural Networks}

\author{%
Ningyuan (Teresa) Huang $^{1}$ \quad Soledad Villar$^{1}$ \quad Carey E. Priebe$^1$ \\
\textbf{\quad Da Zheng$^2$} \quad \textbf{Chengyue Huang}$^3$ \quad \textbf{Lin Yang}$^4$ \quad \textbf{Vladimir Braverman}$^5$ \\
$^1$Johns Hopkins University \quad $^2$Amazon \quad $^3$Renmin University of China \\
$^4$University of California, Los Angeles \quad $^5$Rice University
}

\begin{document}

\maketitle

\begin{abstract}
Graph Neural Networks (GNNs) are powerful deep learning methods for Non-Euclidean data. Popular GNNs are message-passing algorithms (MPNNs) that aggregate and combine signals in a local graph neighborhood. However, shallow MPNNs tend to miss long-range signals and perform poorly on some heterophilous graphs, while deep MPNNs can suffer from issues like over-smoothing or over-squashing. To mitigate such issues, existing works typically borrow normalization techniques from training neural networks on Euclidean data or modify the graph structures. Yet these approaches are not well-understood theoretically and could increase the overall computational complexity. In this work, we draw inspirations from spectral graph embedding and propose \texttt{PowerEmbed} --- a simple layer-wise normalization technique to boost MPNNs. We show \texttt{PowerEmbed} can provably express the top-$k$ leading eigenvectors of the graph operator, which prevents over-smoothing and is agnostic to the graph topology; meanwhile, it produces a list of representations ranging from local features to global signals, which avoids over-squashing. We apply \texttt{PowerEmbed} in a wide range of simulated and real graphs and demonstrate its competitive performance, particularly for heterophilous graphs.
\end{abstract}

\section{Introduction}

We consider the graph representation learning problem. Given the graph $G$ of $n$ nodes with adjacency matrix $A \in \R^{n \times n}$ and node features $X \in \R^{n \times p}$, a function $f$ embeds the graph $G$ by mapping $(A,X)$ to $\mathbf{h} \in \R^{n \times k}$, where each node is associated with an embedding vector $h_i \in \R^k$. The task is to find a permutation equivariant function $f$ (i.e. $f(\Pi A \Pi^{\top}, \Pi X) = \Pi f(A, X)$ for all  permutation matrices $\Pi$), while preserving the graph structure: similar nodes in $G$ should be close to each other in the embedding space $\R^k$. The quality of the embedding $\mathbf{h}$ is typically evaluated in the subsequent statistical inference task, such as node classification or link prediction.

Spectral embeddings choose $f$ based on matrix factorization. For example, $f$ could output the top-$k$ eigenvectors of $A$. 
Spectral embeddings are well-established statistical methods that enjoy nice theoretical properties on random graphs, e.g. stochastic block models \cite{Lei_2015}, latent position random graphs \cite{athreya2017statistical}. Nonetheless, spectral methods typically employ a global spectral decomposition, which can be computationally prohibitive on large-scale graphs; they are also brittle on sparse graphs due to the bias in graph spectra estimation \cite{abbe2017community},

Graph neural networks (GNNs) parameterize $f$ as a neural network, and typically optimize $f$ with label supervision from the subsequent inference task. Popular GNNs are message-passing algorithms (MPNNs) that aggregate and combine signals in a local graph neighborhood. Despite their computational efficiency and scalability, recent works have shown that shallow MPNNs tend to miss long-range signals or global information \citep{alon2021on}, while deep MPNNs can suffer from issues like over-smoothing (i.e., losing information for the inference task) \citep{li2018deeper, oono2019graph} or over-squashing (i.e. distortion of information from distant nodes) \citep{alon2021on, topping2021understanding}. Moreover, MPNNs seem to perform poorly on several \textit{heterophilous} graphs (i.e., nodes from different classes are more likely to connect, as opposed \textit{homophilous} graphs) \cite{Pei2020Geom-GCN,zhu2020beyond, ma2022is}. To mitigate such issues, existing MPNNs typically borrow optimization techniques from training standard neural networks \citep{li2019deepgcns, cai2021graphnorm} or modify the graph structures \citep{velic2018graph, xu2018jump} with increasingly complicated architectures \citep{ying2021transformers, kreuzer2021rethinking}. Nevertheless, these approaches do not fully exploit the nature of the graph learning problem. To encode the right kind of inductive bias, there are two major classes of theoretically-driven GNNs: (1) higher-order GNNs that are based on the Weisfeiler-Lehman hierarchy \cite{HuangWL2021, sato2020survey, xu2018powerful, morris2019weisfeiler, morris2020weisfeiler, bodnar2021weisfeiler, cotta2021reconstruction, chen2019equivalence, chen2020can, chris2022wl}; (2) spectral GNNs that are motivated from graph Fourier transform \citep{bruna2013spectral, defferrard2016convolutional, perlmutter2019understanding, liao2019lanczosnet, ruiz2021}. Despite their nice theoretical properties, they typically incur higher computational costs (e.g., due to operations on higher-order tensors in higher-order GNNs, or additional computations to parameterize the spectral filter functions in spectral GNNs). 

\subsection{Our Contributions} \label{sec:contribution}
Inspired by spectral embeddings, we propose a simple normalization technique in MPNNs to encode global spectra information. Specifically:

\begin{enumerate}
    \item We propose an unsupervised representation learning method \texttt{PowerEmbed} (Algorithm \ref{alg:power_gnn}) by augmenting the message-passing layer with a simple normalization step; \texttt{PowerEmbed} can express the top-$k$ eigenvectors of the graph operator, which is agnostic to the graph topology and works well for graphs with homophily or heterophily.
    \item We couple \texttt{PowerEmbed} with an inception network (Algorithm \ref{alg:sign}) to learn the rich representations that interpolate from local message-passing features to global spectral information, which provably avoids over-smoothing and over-squashing.
    \item We demonstrate numerically that our simple techniques achieve competitive performance for node classification in a wide range of simulated and real-world graphs.
\end{enumerate}

\section{Related Work}


\textbf{Normalization techniques in MPNNs.} Many normalization techniques have been proposed to design deeper MPNNs: \cite{pairnorm, yang2020revisiting} proposed a layer-wise mean subtraction and rescaling to maintain pairwise node embedding distances; \cite{zhou2020towards} proposed a layer-wise node normalization depending on their (predicted) clusters; \cite{cai2021graphnorm, chen2022learning} adapted normalization techniques from deep learning and propose node-wise, batch-wise, and graph-wise normalization methods.


\textbf{Injecting global information in MPNNs.} Global properties of the graph can be encoded as inputs to MPNNs, such as using spectral embeddings as node features \cite{kreuzer2021rethinking, Dwivedi2022learnablePE},  sampling anchor nodes \cite{You2019anchor}, or using other low-pass geometric features \cite{wenkel2022overcoming}. On the other hand, the global information may be learned by using specific architectural choices, such as residual connections \cite{li2019deepgcns, chen-plmr}, attention mechanisms \cite{xu2018jump, geniepath, liu2021non, xie2020reinceptione}, or transformers \cite{ ying2021transformers, kreuzer2021rethinking, wu2021representing}. Furthermore, to speed up the long-range information flow, the original graph can be modified, such as graph sparsification \cite{rong2020dropedge, li2020sgcn, chen2021unified},  graph sampling \cite{zeng2021deep, yoon2021performance}, or localized subgraph extraction \cite{zeng2021decoupling}. Finally, multiscale graph representation learning can be used, such as scattering transform \cite{gama2018diffusion, perlmutter2019understanding}, or hierarchical approaches \cite{mousavi2017hierarchical, ying2018hierarchical}.  

\textbf{GNNs for community detection.} Several GNNs have been proposed to tackle community detection in stochastic block models. These include semi-supervised approaches such as line-GNNs \cite{chen2017supervised}, diff-pooling \cite{ying2018hierarchical}, and GNNs combined with Markov random field models \cite{jin2019graph}. There are also unsupervised approaches such as graph-pooling \cite{bianchi20graphpool}.  

To the best of our knowledge, the statistical properties of the aforementioned MPNNs are not well-understood. Moreover, many of them incur high computational costs. This motivates our work to propose simple techniques that enjoy provable statistical guarantees and computational efficiency.

\section{Preliminaries} \label{sec:prelim}

\textbf{Notations. } Let the graph $G = (V, E, X, Y)$ with node set $V$, edge set $E$, node input features $X \in \R^{n \times p}$ and node labels $Y \in \R^{n}$, where $|V| = n$. Let $A, D$ be the adjacency matrix and the degree matrix of $G$. Let $\bar{A} = \tilde{D}^{-0.5} \tilde{A} \tilde{D}^{-0.5}$ be the symmetric graph Laplacian, where $\tilde{A} = A + I, \tilde{D} = D + I$. Let $A_{rw} = \tilde{D}^{-1} \tilde{A}$ be the random walk graph Laplacian. Let $ \frac{U}{\| U[:,k]\|}$ be the column normalized form of matrix $U$. 

\textbf{Inference problem setup.} To evaluate the graph embedding $\mathbf{h} = f(A, X)$, we consider the node classification task: Given a graph $G$ of $n$ nodes with $n_t < n$ training labels, the goal is to predict the remaining $n - n_t$ test labels. Let $\Omega$ and $\Omega^\perp$ the indices corresponding to training nodes and test nodes respectively ($V=V_\Omega \cup V_{\Omega^\perp}$). Then the classifier is trained on the training set  $(\mathbf{h}_{\Omega}, Y_{\Omega})$ and evaluated on the test set $(\mathbf{h}_{\Omega^{\perp}}, Y_{\Omega^\perp})$. If the graph embedding is a list of features $P = [X,\,\textbf{h}^{(1)} \ldots, \textbf{h}^{(L)}]$, then we let $P_{\Omega}, P_{\Omega^\perp}$ denote the training and test set features, where
\begin{equation}
    P_{\Omega} = [X_{\Omega}, \textbf{h}^{(1)}_{\Omega}, \ldots, \textbf{h}^{(L)}_{\Omega}]. \label{eqn:ptrain}
\end{equation}

\begin{defn}[Stochastic Block Model (SBM)] \label{defn:sbm}
A graph $A$ with $n$ nodes is a random SBM graph if it is sampled as 
\begin{equation}
A \sim Bernoulli(P), \, 
P = Z B Z^{\top}, \, 
\label{eqn:SBM}
\end{equation}
where $Z \in \R^{n \times K}$ is a membership matrix such that $Z_{i,k}$ is 1 if the $i$-th node belongs to the $k$-th class, $\| Z_{i,\cdot} \|_1 = \sum_{k=1}^K |Z_{i,k}| = 1$, and $B \in [0,1]^{K \times K}$ is a full-rank matrix representing the block connection probability. 
\end{defn}

\begin{defn}[2B-SBM with Gaussian node features] \label{den:csbm}
A two-block symmetric SBM (2B-SBM) is given by:
\begin{equation}
    Z_{i,\cdot} = \begin{cases*}
    [1,0] & if  $i \in [n/2]$  \\
    [0,1] & otherwise.
    \end{cases*}, \, \,
    B =  \begin{bmatrix}
p  & q\\
q  & p
\end{bmatrix},  \label{eqn:2B-SBM}
\end{equation}
where $p, q \in (0,1), p \ne q$. 
The node features in block $k \in \{0,1\}$ are sampled from a $m$-dimensional multivariate Gaussian $\mathcal{N}(\mu_k, \Sigma_k)$ and stored in a node feature matrix $X \in \R^{n \times m}$.
\end{defn}

The SBM model is a canonical random graph model with planted clusters, which has been widely studied in the context of community detection \cite{abbe2017community, athreya2017statistical,Lei_2015, lyzinski2014perfect}. In particular, the leading eigenvectors of the graph adjacency $A$ or the graph Laplacians $A_{rw}, \bar{A}$ provably encode community membership, provided that the graphs are sufficiently dense \cite{von2008consistency, Lei_2015}. To see this, consider the 2B-SBM model in \eqref{eqn:2B-SBM} without node features, then the matrix $P = \E(A)$ has three eigenvalues, ranked by magnitude as $(p+q)/2, (p-q)/2, 0$, where $0$ has multiplicity $n-2$. The leading eigenvector is a constant vector, whereas the second eigenvector $u_2(P) = [\mathbf{1}_{n/2}; \mathbf{-1}_{n/2}]$ reveals the community structure. Now, the random graph $A$ can be viewed as a perturbation of its expectation $P$. Standard results in concentration of measure show that the eigenvalues and eigenvectors of $A$ are close to those of $P$ for sufficiently dense graphs (the average node degree needs to be of the order $\Omega(\frac{n}{\log n})$). However, spectral methods fail to detect communities when the graphs are very sparse ($p=q=O(1/n)$), and there are known information-theoretical thresholds for certain stochastic block models (see \cite{abbe2017community} for a recent survey). 

In this work, we consider classical spectral embedding methods defined as follows: Let $A = U S U^{\top}, X X^{\top} = \tilde{U} \Sigma \tilde{U}^{\top} $ be the spectral decomposition of the graph $A$ and the covariance matrix $ X X^{\top}$. Then
\begin{equation}
    \mathbf{h}^{ASE} = U_k, \quad \mathbf{h}^{\operatorname{cov}(X)} = \tilde{U}_k , \quad \mathbf{h}^{A\_X} =  [\mathbf{h}^{ASE} ;  \mathbf{h}^{\operatorname{cov}(X)}]  \label{eqn:ASE}.
\end{equation}

Crucially, spectral methods that are based on the top-$k$ eigenvectors ranked by absolute \textit{magnitude} have no assumption on the graph being homophilous (e.g., $p>q$) or heterophilous (e.g., $p<q$); They perform consistently well on both cases by ``looking globally'' at the graph structure. In contrast, message-passing methods are local algorithms that work well in homophilous graphs, but they tend to overlook global signals and struggle to learn in heterophilous graphs.

\begin{defn}[Message-Passing Neural Network (MPNN)]
A $L$-layer MPNN initializes the embedding $\mathbf{h}^{(0)} = X$. At each iteration $l$, the embedding of node $i$ is updated as
\begin{equation}
    \mathbf{h}^{(l)}_i = \phi \left( \mathbf{h}^{(l-1)}_i,   \sum_{j \in \mathcal{N}(i)} \psi(\mathbf{h}^{(l-1)}_i, \mathbf{h}^{(l-1)}_j) \right), \label{eqn: MPNN}
\end{equation} 
where $\phi, \psi$ are the update and message functions, and $\mathcal{N}(i)$ denotes the neighbors of node $i$. 
\end{defn}

For example, graph convolutional networks (GCN \cite{kipf2016semi}) choose $\phi$ as a nonlinear activation function $\sigma$ and $\psi(i,j) = \frac{1}{ \sqrt{\operatorname{deg}(i) \operatorname{deg}(j)}} W^{(l-1)} \mathbf{h}^{(l-1)}_j $, where $W^{(l-1)}$ is a learnable weight matrix. In other words:
\begin{equation}
   \mathbf{h}^{(l)}  = \sigma ( \bar{A} \mathbf{h}^{(l-1)}  W^{(l-1)}). \label{eqn: GCN}
\end{equation}

GCN can be further simplified by removing the nonlinearity $\sigma$, which is known as simple graph convolution (SGC \cite{wu2019simplifying}):
\begin{equation}
   \mathbf{h}^{(L)}  = \bar{A} \mathbf{h}^{(L-1)}  W^{(L-1)} =  \bar{A}^{L} X \big(W^{(L-1)} \ldots W^{(0)} \big). \label{eqn: SGC}
\end{equation}
To train a classifier on top of the SGC embedding in \eqref{eqn: SGC}, we can absorb the weight matrices $W^{(l-1)}$ to the downstream classifier. In other words, SGC can be viewed as producing an unsupervised representation $\bar{A}^{L} X$ followed by learning a classifier. However, such representation fails to distinguish nodes from different classes when $L$ is sufficiently large, known as the over-smoothing issue of deep GNNs. 

\begin{defn}[Over-smoothing \cite{li2018deeper}] \label{defn:over-smoothing}
Assume the graph has one connected component. Then 
\begin{equation}
\lim_{L \to \infty} A_{rw}^{L} X = [\mathbf{1}_{n}; , \ldots, \mathbf{1}_{n}]; \quad \lim_{L \to \infty} \bar{A}^{L} X = D^{-0.5} [\mathbf{1}_{n}; , \ldots, \mathbf{1}_{n}]. \label{eqn:over-smoothing}    
\end{equation}
\end{defn}
Over-smoothing describes that for a sufficiently deep SGC model, its node embedding converges to the largest eigenvector which only encodes connected component and degree information, but tends to miss the community structure present in subsequent eigenvectors. It has also been extended and defined for nonlinear GCNs \cite{oono2019graph, cai2020note} and empirically observed in deep MPNNs \cite{pairnorm, rong2020dropedge, chen-plmr}. 

On the other hand, over-squashing refers to the distortion of node features in the learned graph representation in a deep GNN model \cite{alon2021on, topping2021understanding}, which can also hurt subsequent inference when the node features carry useful signals.
\begin{defn}[Over-squashing \cite{topping2021understanding}] \label{defn:over-squashing}
Let $h_i^{(L)} = h_i^{(L)}(x_1, \ldots, x_n)$ be the output for node $i$ of a $L$-layer MPNN with input features $\{ x_i\}_{i=1}^n$. Then the over-squashing effect (for node $i$ with respect to node $s$) is measured by the Jacobian $\partial h_i^{(L)} / \partial x_s$. 
\end{defn}
Thus, the smaller the Jacobian value, the more node feature information is ``squashed'' out in the embedding.

\section{PowerEmbed} \label{sec:powerembed}

In this section, we first show how to express the top-$k$ eigenvectors of a graph operator using message-passing power iteration (\texttt{PowerEmbed} in Algorithm \ref{alg:power_gnn}). Next, we describe how to exploit the rich representations produced by \texttt{PowerEmbed} by using the inception neural network architecture (Algorithm \ref{alg:sign}). Our framework is summarized in Figure \ref{fig:power_overview}.

\begin{figure}[htb!]
  \centering
    \includegraphics[width=0.75\textwidth]{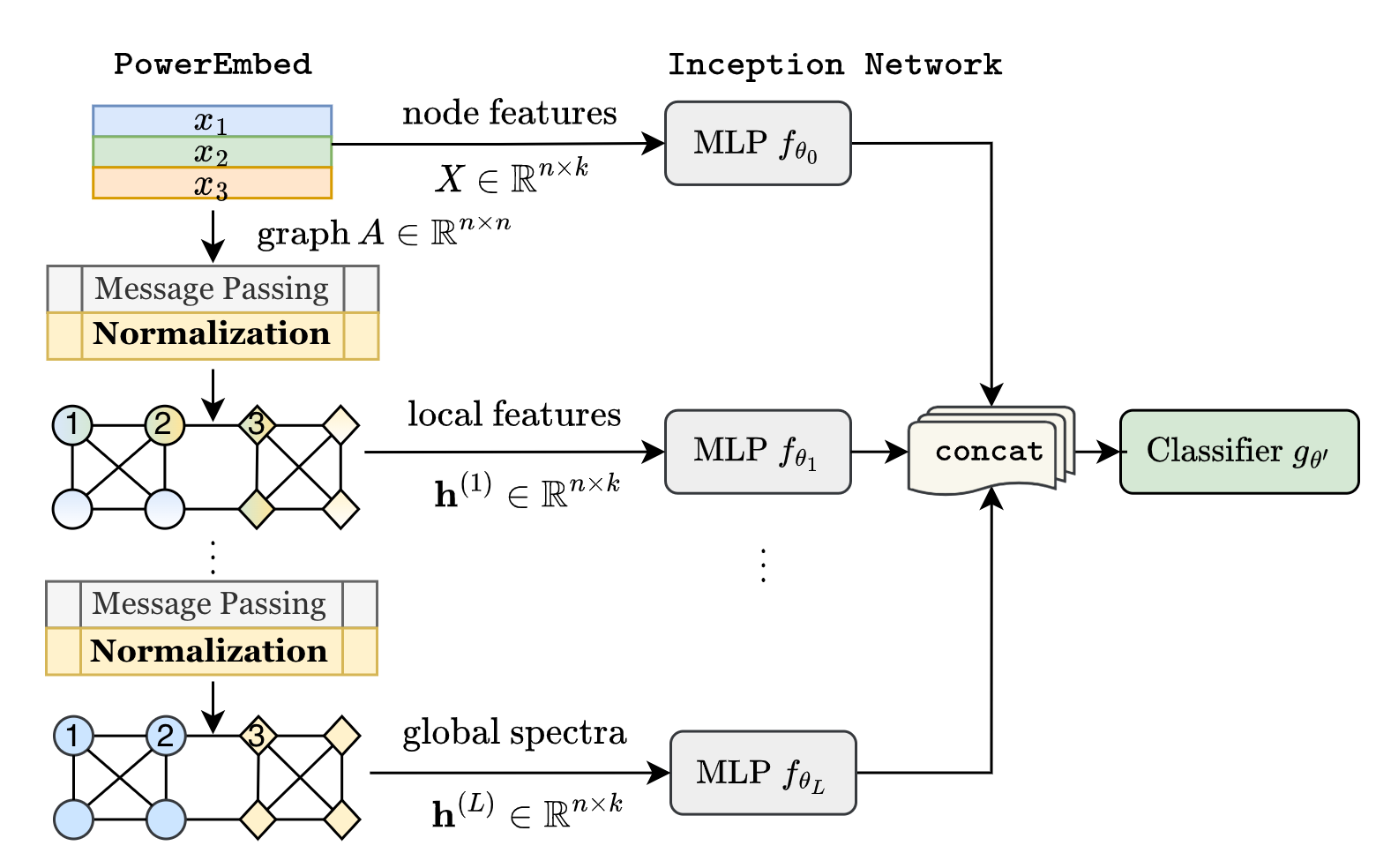}
  \caption{\texttt{PowerEmbed} produces a list of embeddings ranging from local neighbor-averaged features to global spectral information, which is then jointly learned by using an inception neural network.} 
  \label{fig:power_overview}
\end{figure}

\begin{algorithm}[htb!]
\caption{PowerEmbed}\label{alg:power_gnn}
\begin{algorithmic}
\Require a graph operator $S \in \R^{n \times n}$, node features $X \in \R^{n \times k}$, a list $P = [X]$.

\Initialize{$U(t)=X$}
\For{t = 0 to L-1}
\State{$\tilde{U}(t+1) = S \, U(t)$ [message-passing]} 
\State{$U(t+1) = \tilde{U}(t+1)  [ \tilde{U}(t+1) ^{\top}  \tilde{U}(t+1)  ]^{-1}$ [Normalization]} 
\State{Append $ \frac{U(t+1)}{\| U(t+1)[:,k] \|}$ to $P$}
\EndFor \\
\Return $P$ 
\end{algorithmic}
\end{algorithm}

Algorithm \ref{alg:power_gnn} is motivated from the natural power iteration \cite{hua1999new}, which provably returns the top-$k$ eigenvectors of a symmetric matrix $S$ (up to an orthogonal transformation) upon convergence, with mild assumptions that the input features $X$ has full column rank and the $S$ has an eigen-gap, i.e., $\lambda_k \ne \lambda_{k+1}$ where $\lambda_k$ denotes the $k$-th eigenvalue of $S$ (ranked by magnitude). Crucially, the normalization step ensures orthogonality of the output at each iteration and only requires inverting a $k\times k$ matrix (which may be further improved using variants of QR factorization). The convergence rate of the $k$-th eigenvector depends on $\frac{\lambda_{k+1}}{\lambda_k}$, which is fast provided with a large eigen-gap $|\lambda_{k}| \gg |\lambda_{k+1}|$. In practice, if the original node features live in high-dimensional space (e.g., more features than the number of nodes), we first perform a dimensionality reduction to obtain the best rank-$k$ approximation of $X$ before running Algorithm \ref{alg:power_gnn} to speed up convergence and reduce memory complexity. 

Unlike existing work that requires pre-computed eigenvectors as input \cite{kreuzer2021rethinking, Dwivedi2022learnablePE, lim2022sign}, our approach implicitly computes eigenvectors using power iteration to improve efficiency while maintaining expressivity. By embedding the original graph to a list of Euclidean features, \texttt{PowerEmbed} allows fast training and inference with computational complexity independent of the graph topology. \texttt{PowerEmbed} extracts both \textit{local} features from the first few iterations and \textit{global} information from the last few iterations (i.e., the top-$k$ eigenvectors).
 These rich features are jointly learned using the inception network (Algorithm \ref{alg:sign}). Therefore, the presence of local features (including the input features) can avoid over-squashing in Definition \ref{defn:over-squashing}; the ability to express top-$k$ eigenvectors of a graph operator alleviates over-smoothing in Definition \ref{defn:over-smoothing} and makes it agnostic to the graph topology. Such expressivity is necessary to avoid over-smoothing/over-squashing considered in this work, but may not be sufficient depending on the data and the optimization procedure.


\begin{algorithm}[htb!]
\caption{Node Classification with Inception Networks on \texttt{PowerEmbed} Representations} \label{alg:sign}
\begin{algorithmic}
\Require Training features $P_{\Omega}$ \eqref{eqn:ptrain}, training labels $Y_{\Omega}$, loss function $\ell$, step size $\eta$, hidden dimension $k'$, number of classes $K$,  epochs $T$.
\Initialize{A list of Multilayer Perceptions $\texttt{MLPs} = [f_{\theta_0}, \ldots, f_{\theta_{L}}]$ where $f_{\theta_i}: \R^k \to \R^{k'}$, a classifier $g_{\theta'}: \R^{(L+1)k'} \to \R^K$.} 
\For{t = 0 to T}
\State{$H = \Big[ f_{\theta_0}\left(P_{\Omega}[0] \right), \ldots,  f_{\theta_L}\left(P_{\Omega}[L] \right) \Big]$}
\State{$H' = \texttt{concat}(H) \in \R^{m \times (L+1)k'}$} 
\State{$\hat{Y} = g_{\theta'}(H')$}
\State{$L = \ell(\hat{Y}, Y_{\Omega})$}
\State{$\theta = \theta - \eta \nabla_{\theta}L$ for all $\theta \in \{\theta_0, \ldots, \theta_{L}, \theta' \}$} 
\EndFor \\
\Return $g_{\theta'} \circ \texttt{MLPs}$
\end{algorithmic}
\end{algorithm}






The following propositions, proven in Appendix \ref{app.powerembed} are basic facts of graph theory. Proposition \ref{prop.limit} applies the convergence results from \cite{hua1999new} to the graph operators used in Algorithm \ref{alg:power_gnn}, and it shows that \texttt{PowerEmbed} can express the top-$k$ eigenvectors of the common graph operators. Proposition \ref{prop.equivariance} shows that our graph representation learning approach is permutation equivariant \cite{maron2018invariant}.
\begin{prop}
\label{prop.limit}
Consider a graph $G$ with graph operator $S \in \{A, \bar{A}, A_{rw}\}$ and feature matrix $X \in \R^{n \times k}$. If $X$ is full (column) rank and the $k$-th and $(k+1)$-th eigenvalues of $S$ are distinct, then the last iterate $U(t+1)$ from \texttt{PowerEmbed} converges to the top $k$ eigenvectors of $S$ when $t\to \infty$ (up to an orthogonal transformation in $O(k)$). 
\end{prop}

\begin{prop}\label{prop.equivariance}
The method \texttt{PowerEmbed} (Algorithm \ref{alg:power_gnn}) followed by the inception neural network $g_{\theta'} \circ \texttt{MLPs}$ (from Algorithm \ref{alg:sign}) is permutation equivariant.
\end{prop}

\section{Numerical Experiments}\label{sec:experiment}

\subsection{Baselines}

To evaluate the benefits of injecting global spectra information into MPNNs, We compare \texttt{PowerEmbed} with three classes of baselines.

\textit{Unnormalized counterparts.} Unsupervised MPNNs without layer-wise normalization using the same inception network architecture (Algorithm \ref{alg:sign}) : SIGN \cite{frasca2020SIGN} that uses $A_{rw}$, and a modified version of SGC \cite{wu2019simplifying} that uses $\bar{A}$ which we call ``SGC(Incep)''. These baselines illustrate the effects of adding the proposed layer-wise normalization in Algorithm \ref{alg:power_gnn}.

\textit{Spectral methods.} Standard spectral embedding methods described in \eqref{eqn:ASE} that use purely global (spectral) information. These global baselines measure the effects of adding local information (from neighbor-aggregated features) in \texttt{PowerEmbed}.

\textit{Semi-supervised MPNNs.} Popular models including graph convolutional networks (GCN) \cite{kipf2016semi} and graph attention networks (GAT) \cite{velivckovic2017graph}) for benchmark purpose; MPNNs that encode long-range information including geometric GCN (GEOM-GCN) \cite{Pei2020Geom-GCN}, generalized page-rank GNN (GPR-GNN) \cite{chien2020adaptive}, GCN via initial residual and identity mapping (GCNII) \cite{chen2020simple}, and Jumping Knowledge Networks with layer concatenation (JK-Concat) \cite{xu2018jump}; These baselines are chosen to gauge the effects of injecting \textit{spectral} information as opposed to long-range spatial information.

We remark that all unsupervised methods (i.e., \texttt{PowerEmbed}, unnormalized counterparts, spectral methods) produce an embedding matrix $\textbf{h} \in \R^{n \times k}$ (or a list of features $P$) without label supervision, followed by fitting a neural network classifier: For a single embedding matrix $\textbf{h}$, we fit a 2-layer MLP with ReLU activation; for the list $P = [X, \textbf{h}_1, \ldots, \textbf{h}_L]$, we fit the inception network (Algorithm \ref{alg:sign}). On the other hand, all semi-supervised MPNNs optimize the network end-to-end.

All the simulations and code are available in this repository \footnote{https://github.com/nhuang37/spectral-inspired-gnn}. 

\subsection{Synthetic graphs} \label{subsec: sbm}
In this section, we simulate random graphs following the stochastic block models defined in \eqref{eqn:2B-SBM}, with node features $X$ sampled from a mixture of two Gaussians with mean $\mathbf{\mu}_0 = [1,1],  \mathbf{\mu}_1 = - \mathbf{\mu}_0$. 
We consider the (binary) node classification setting described in Section \ref{sec:prelim} with $n=500$ nodes and a 10/90 train/test split. We choose $k=2$ (i.e., use $X$ without dimensionality reduction).

Figure \ref{fig:convergence} illustrates the convergence of the last iterate of \texttt{PowerEmbed} to the top-$2$ eigenvectors of a 2B-SBM expected adjacency matrix, with different choices of the parameters $p,q$ that controls the block probability matrix $B$ defined in \eqref{eqn:2B-SBM}. The convergence is fast in  the dense case where $\lambda_i(A) \approx \lambda_i(P), i \in \{1,2\}$ and are well-separated from the remaining $n-2$ eigenvalues. In the sparse case, the eigen-gap $\lambda_2(A) - \lambda_3(A)$ is smaller due to higher (random) noise in the sparse random graph $A$, and thus the convergence of $\hat{u}_2$ is slower.

\begin{figure}[htb!]
  \centering
  \includegraphics[width=0.8\textwidth]{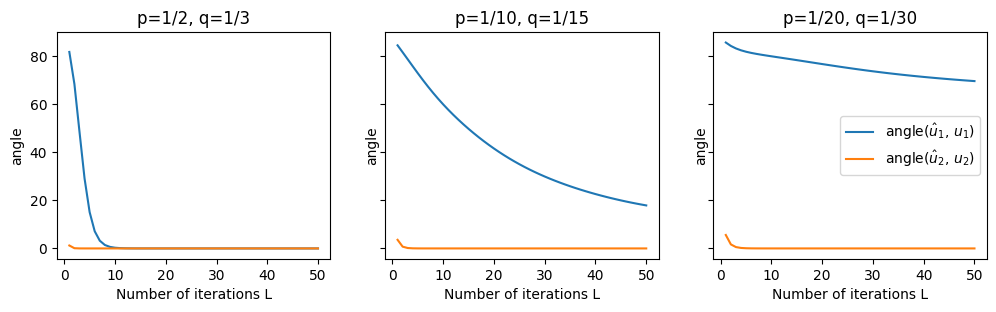}
  \caption{Convergence of the estimated eigenvectors $\hat{u}_i$ from Algorithm \ref{alg:power_gnn} is faster for denser graphs and slower for sparse graphs. We report the angles between the true top-$2$ eigenspaces and estimated top-$2$ eigenspaces, averaged over 30 random runs.}
  \label{fig:convergence}
\end{figure}

Figure \ref{fig:power-sbm-simulation} shows that \texttt{PowerEmbed} is able to recover the spectral embeddings in dense graphs (left column), and outperforms spectral embedding in sparse graphs (right column). This can be explained by \texttt{PowerEmbed}'s ability to leverage both local node features and global structural information. Moreover, \texttt{PowerEmbed} and spectral methods perform consistently well on graphs with homophily (top row) and heterophily (bottom row). This is in contrast to standard MPNNs that suffer from over-smoothing (e.g., ``GCN-10'') and fail in sparse heterophilous graphs (e.g., ``GCN-5'' performs much worse in heterophilous graphs in (d) compared to homophilous graph in (b) while ``Power(Lap)-10'' performs steadily well). Comprehensive simulation results including other baselines can be found in Appendix \ref{app.ablation_sbm}.

\begin{figure}[htb!]
  \centering
  \includegraphics[width=0.9\textwidth]{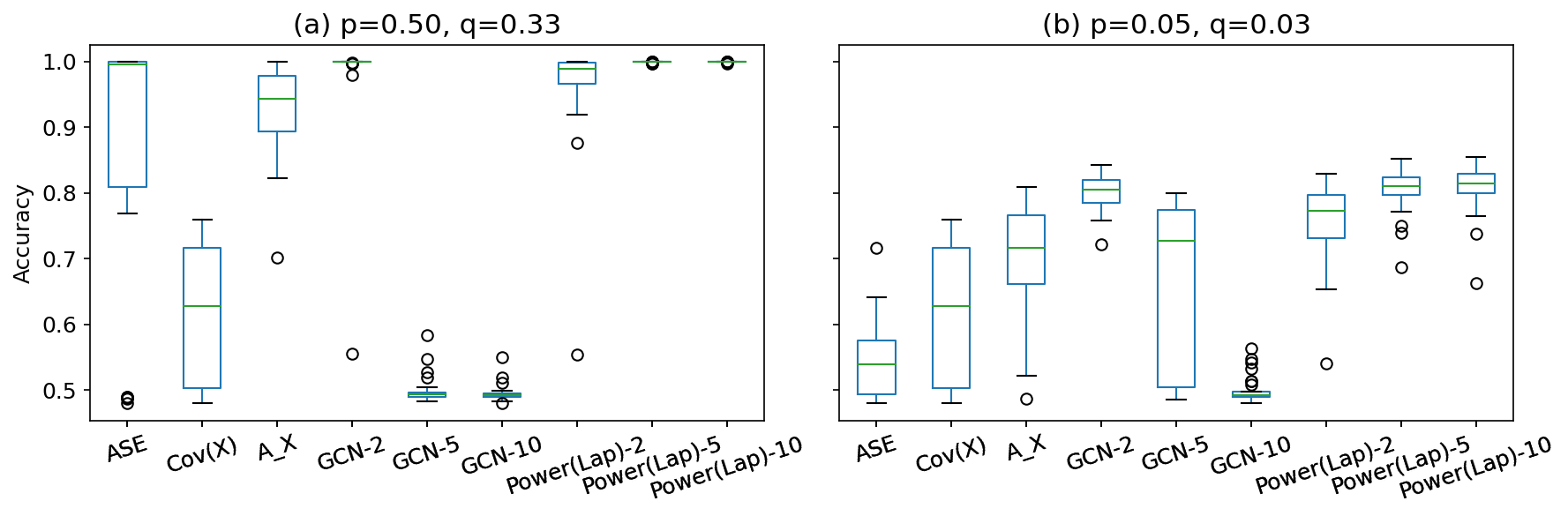}
   \includegraphics[width=0.9\textwidth]{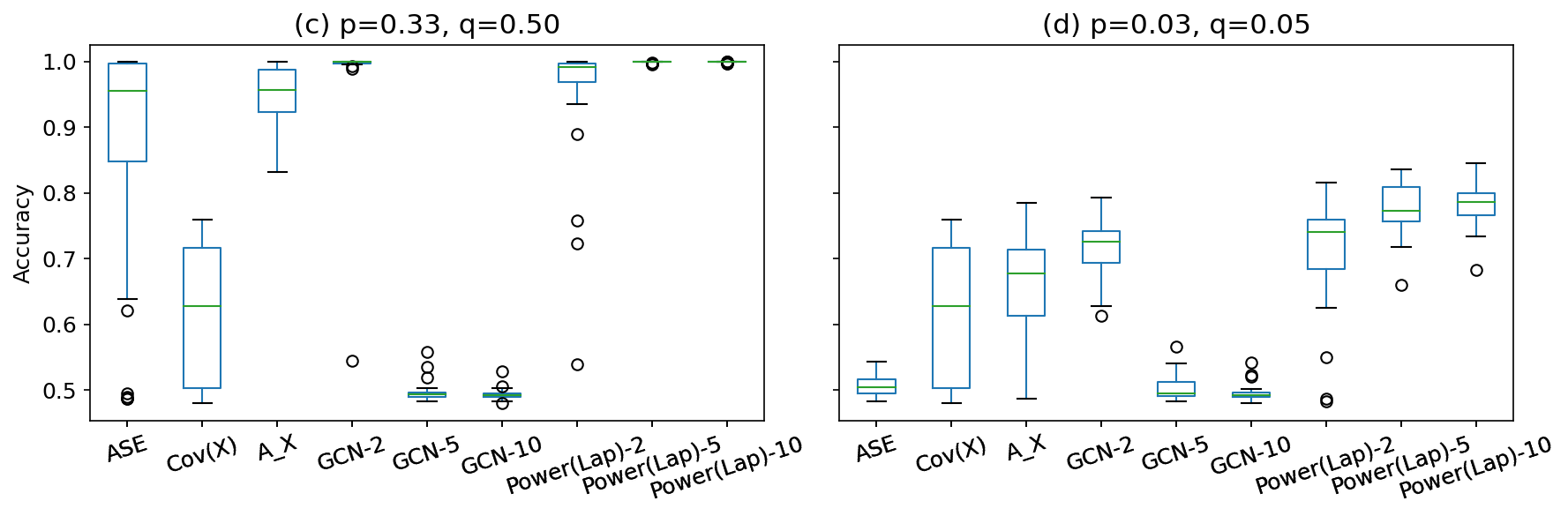}
  \caption{\texttt{PowerEmbed} enjoys the same performance guarantee as spectral embedding methods in dense graphs (left) with homophily (top-left) and heterophily (bottom-left); It outperforms spectral embedding methods in sparse graphs (right). ``ASE'', ``Cov(X)'', ``A\_X'' are global spectral methods defined in \eqref{eqn:ASE}. Standard MPNNs suffer from over-smoothing (e.g., ``GCN-10'') and perform much worse in sparse heterophilous graphs (e.g., ``GCN-5''). ``Method-$k$'' denotes the corresponding method with $k$ iterations. Boxplots are shown for each method, with the mean indicated by green bars and outliers displayed as circles.}
  \label{fig:power-sbm-simulation}
\end{figure}

\subsection{Real-world graphs}

We report node classification results of \texttt{PowerEmbed} on 10 graph benchmark datasets, including 5 heterophily graphs and 5 homophily graphs with varying densities. We summarize key statistics of each graph in Table \ref{tab:experiments}, Table \ref{tab:experiments-sparse} and describe their details in Appendix \ref{app.datasets}. 

For most graphs except Computers and Photo, we use the same data splits (48/32/20 for train/validation/test) from \cite{Pei2020Geom-GCN} released in Pytorch Geometric \cite{Fey2019pytorchgeo}; for Computers and Photo that were not studied in \cite{Pei2020Geom-GCN}, we use the same data split (60/20/20) as in \cite{luo2022inferring}. Note that most of the node features are high-dimensional bag-of-words vectors. Thus, for unsupervised embedding methods and JK-Concat, we use the top-$k$ eigenvector of the feature covariance as the input feature, where $k=10$ for small-size graphs including Wisconsin, Texas, and Cornell, and $k=100$ elsewhere; for semi-supervised methods (GCN-II, GPR-GNN), we use the original node features following the original model architectures. Self-loops are used for all graphs. We train all models using ADAM optimizer \cite{kingma2014adam} and full-batch gradient descent with the same hyperparameters: 100 epochs, 0.01 learning rate, 0.5 dropout rate, and no weight decay. Further details can be found in Appendix \ref{app.baselines}.



Table \ref{tab:experiments} summarizes the experimental results on \textit{heterophilous} graph benchmarks: \texttt{PowerEmbed} outperforms other baselines on graphs with heterophily, particularly on dense heterophilous graphs. Figure \ref{fig:normalize_outperform} visualize the effect of increasing the number of message-passing layers in Squirrel and Chameleon, two dense graphs with heterophily: \texttt{PowerEmbed} exhibits consistent performance in deeper models and avoids over-smoothing; The ability to express global spectra information via simple normalization significantly boosts performance over unnormalized MPNN counterparts in heterophilous graphs. Across all 5 heterophilous graphs, we observe that even simple spectral method A\_X concatenating leading eigenvectors of graph spectra and node covariance achieves competitive performance, drastically outperforming unnormalized MPNNs and semi-supervised MPNNs. This highlights the utility of learning graph representation from first principle. Further ablation studies are shown in Appendix \ref{app.ablation_real}.


\begin{table*}
\scriptsize
\caption{Heterophily graph experiments: \texttt{PowerEmbed} and spectral methods (third block) significantly outperform unsupervised MPNN counterparts (second block) and semi-supervised MPNNs (last block) in dense heterophilous graphs; \texttt{PowerEmbed} works better than spectral methods on sparse heterophilous graphs. We report the mean accuracy $\pm$ stderr over $10$ data splits. ``Method-$k$'' indicates the corresponding method with $k$ iterations. The $*$ results are obtained from \cite{Pei2020Geom-GCN}.} 
\label{tab:experiments}
\centering
\begin{tabular}{cccccc}
\hline \hline
Graph	&	Squirrel	&	Chameleon	&	  Wisconsin 	&	 Texas 	&	 Cornell 	\\
Density	&	38.16	&	13.8	&	1.86	&	1.61	&	1.53	\\
Homophily	&	0.22	&	0.23	&	0.21	&	0.11	&	0.3	\\
\#Nodes	&	5,201	&	2,277	&	251	&	183	&	183	\\
\#Edges	&	198,493	&	31,421	&	466	&	295	&	280	\\
\#Features	&	2,089	&	2,325	&	1,703	&	1703	&	1,703	\\
\#Classes	&	5	&	5	&	5	&	5	&	5	\\
	\hline										
Power-10	&	  \textbf{53.53 $\pm$ 0.41}  	&	 \textbf{64.98 $\pm$ 0.55}  	&	74.71  $\pm$ 1.74	&	 73.51 $\pm$ 2.05 	&	75.14 $\pm$ 2.50	\\
Power(RW)-10	&	  44.58 $\pm$ 0.52	&	61.64 $\pm$ 0.43 	&	75.49  $\pm$ 1.71	&	75.68 $\pm$ 1.21	&	72.97 $\pm$ 1.58	\\
Power(Lap)-10	&	42.32 $\pm$ 0.37	&	62.17 $\pm$ 0.41	&	74.71 $\pm$ 1.74	&	74.05 $\pm$ 2.10	&	77.03 $\pm$ 1.54	\\
Power-2	&	52.13 $\pm$ 0.55	&	64.47 $\pm$ 0.76  	&	 75.29 $\pm$ 1.47	&	\textbf{79.19 $\pm$ 1.33}	&	 76.76 $\pm$ 1.63 	\\
Power(RW)-2	&	  45.92 $\pm$ 0.48	&	 59.67 $\pm$ 0.62	&	77.45 $\pm$ 0.89	&	76.22 $\pm$ 1.31	&	75.41 $\pm$ 1.85	\\
Power(Lap)-2	&	43.06 $\pm$ 0.56	&	60.00 $\pm$ 0.62	&	\textbf{78.43 $\pm$ 1.59}	&	77.03 $\pm$ 1.54	&	\textbf{78.30 $\pm$ 1.58}	\\
	\hline										
SGC(Incep)-10	&	  37.07 $\pm$ 0.55  	&	 55.11 $\pm$ 0.82  	&	75.29 $\pm$ 1.04	&	75.68 $\pm$ 1.95	&	75.68 $\pm$ 1.83	\\
SIGN-10	&	  38.47 $\pm$ 0.42  	&	   60.22 $\pm$ 0.72  	&	 75.29  $\pm$ 1.45	&	 73.51 $\pm$ 2.02	&	 75.68 $\pm$ 1.21	\\
SGC(Incep)-2	&	  35.33 $\pm$ 0.35   	&	  54.19 $\pm$ 0.65    	&	 77.45 $\pm$ 0.89	&	76.76 $\pm$ 1.34	&	76.22 $\pm$ 2.07	\\
SIGN-2	&	  40.97 $\pm$ 0.35  	&	   60.11 $\pm$ 0.97  	&	\textbf{78.43 $\pm$ 1.41}	&	 75.14 $\pm$ 2.02 	&	 76.76 $\pm$ 1.34 	\\
	\hline										
Cov(X)	&	33.12 $\pm$ 0.53	&	44.74 $\pm$ 1.00	&	75.69 $\pm$ 1.25	&	77.30 $\pm$ 1.12	&	77.03 $\pm$ 2.27	\\
ASE	&	41.46 $\pm$ 0.62	&	57.92 $\pm$ 0.77	&	 49.41 $\pm$ 2.09	&	 58.65 $\pm$ 1.79	&	 56.76 $\pm$ 0.66	\\
A\_X	&	49.11 $\pm$ 0.37	&	61.97 $\pm$ 0.76	&	77.84 $\pm$ 1.24	&	76.76 $\pm$ 1.22	&	75.95 $\pm$ 2.28	\\
	\hline										
GCN$*$	&	23.96	&	28.18	&	45.88	&	52.16	&	52.7	\\
GAT$*$	&	30.03	&	42.93	&	49.41	&	58.38	&	54.32	\\
Geom-GCN$*$	&	38.14	&	60.9	&	64.12	&	67.57	&	60.81	\\
GCNII-10 	&	 35.23 $\pm$ 0.50 	&	 49.96 $\pm$ 0.46 	&	 59.02 $\pm$ 1.60 	&	 61.08 $\pm$ 1.49 	&	 48.38 $\pm$ 1.64 	\\
GPR-GNN-10 	&	 34.51 $\pm$ 1.45 	&	 52.37 $\pm$ 3.43 	&	 59.41 $\pm$ 2.75 	&	 58.92 $\pm$ 2.98 	&	 52.97 $\pm$ 3.11 	\\
JK-Concat-10 & 40.48 $\pm$ 0.42 & 58.31 $\pm$ 0.38 & 72.75 $\pm$ 1.48 & 73.51 $\pm$ 1.90 & 67.57 $\pm$ 1.43 \\
\hline \hline
\end{tabular}
\end{table*}

\begin{figure}[htb!]
  \centering
  \includegraphics[width=0.95\textwidth]{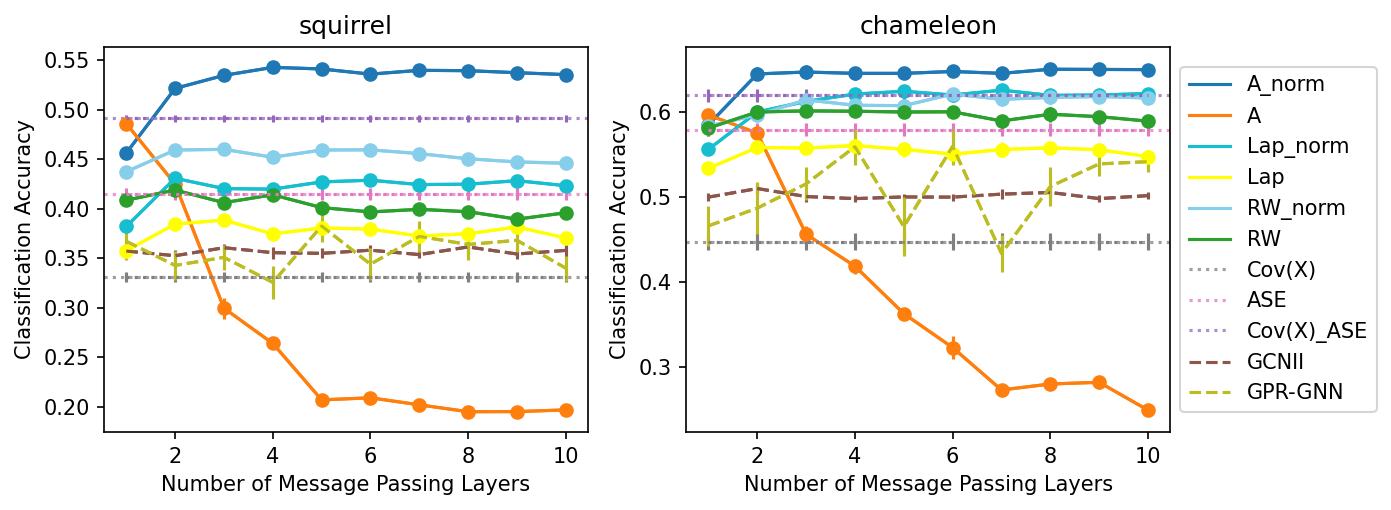}
  \caption{\texttt{PowerEmbed} (annotated with ``\_norm'') that adds the normalization step for orthogonality can expressive top-$k$ eigenvectors, which avoids over-smoothing and outperforms other baselines, particularly in heterophilous graphs. Baselines include unnormalized counterparts (SIGN denoted as ``RW'', SGC(Incep) denoted as ``Lap''); spectral methods defined in \eqref{eqn:ASE}, and semi-supervised MPNNs.}
  \label{fig:normalize_outperform}
\end{figure}

Table \ref{tab:experiments-sparse} summarizes the results on \textit{homophilous} graph benchmarks: \texttt{PowerEmbed} achieves competitive performance as other MPNN baselines. As discussed in Section \ref{sec:prelim}, MPNNs tend to work well in graphs with homophily due to the inductive bias of aggregating local neighbors' information, while spectral methods could perform sub-optimally, especially in sparse graphs. Together with Table \ref{tab:experiments}, \texttt{PowerEmbed} is shown to be agnostic to the graph topology and can flexibly combine local and global information.

\begin{table*}
\scriptsize
\caption{Homophilous graph experiments: \texttt{PowerEmbed} shows competitive performance as other MPNN baselines. We report the mean accuracy $\pm$ stderr over $10$ data splits. ``Method-k'' indicates the corresponding method with $k$ iterations. The $*$ results are obtained in \cite{Pei2020Geom-GCN}, from \cite{gasteiger_diffusion_2019} for Computers and Photo (``-'' denotes results not available).}
\label{tab:experiments-sparse}
\centering
\begin{tabular}{cccccc}
\hline \hline
Graph	&	Computers	&	Photo	&	Coauthor(CS)	&	 Cora 	&	 Citeseer   \\ 	
Density	&	35.76	&	31.13	&	8.93	&	1.95	&	1.41	\\
Homophily	&	0.8	&	0.85	&	0.83	&	0.81	&	0.74	\\
\#Nodes	&	13,752	&	7,650	&	18,333	&	2,708	&	3327	\\
\#Edges	&	491,722	&	238,162	&	163,788	&	5,278	&	4676	\\
\#Features	&	767	&	745	&	6,805	&	1,433	&	3703	\\
\#Classes	&	10	&	8	&	15	&	7	&	6	\\
	\hline										
Power-10	&	90.34 $\pm$ 0.22	&	93.84 $\pm$ 0.17	&	93.93 $\pm$ 0.11	&	81.69 $\pm$ 0.50	&	69.67 $\pm$ 0.63   \\	
Power(RW)-10	&	91.14 $\pm$ 0.19	&	94.16 $\pm$ 0.20	&	93.88 $\pm$ 0.13	&	85.03 $\pm$ 0.44	&	73.15 $\pm$ 0.54 \\	
Power(Lap)-10	&	91.20 $\pm$ 0.14	&	93.97 $\pm$ 0.19	&	94.26 $\pm$ 0.09	&	84.95 $\pm$ 0.40	&	72.61 $\pm$ 0.51	\\
Power-2	&	90.85 $\pm$ 0.15	&	94.04 $\pm$ 0.21	&	94.32 $\pm$ 0.11	&	81.23 $\pm$ 0.52	&	72.03 $\pm$ 0.41   \\	
Power(RW)-2	&	\textbf{91.43 $\pm$ 0.13}	&	94.56 $\pm$ 0.19	&	94.30 $\pm $0.08	&	83.56 $\pm$ 0.44	&	  72.62 $\pm$ 0.48 \\	
Power(Lap)-2	&	91.33 $\pm$ 0.15	&	94.58 $\pm$ 0.21	&	94.75 $\pm$ 0.09	&	83.52 $\pm$ 0.27	&	73.27 $\pm$ 0.75	\\
	\hline										
SGC(Incep)-10	&	90.61 $\pm$ 0.16	&	94.65 $\pm$ 0.18	&	94.44 $\pm$ 0.10	&	84.89 $\pm$ 0.71	&	  73.39 $\pm$ 0.62  \\	
SIGN-10	&	90.65 $\pm$ 0.18	&	94.63 $\pm$ 0.25	&	94.05 $\pm$ 0.13	&	85.45 $\pm$ 0.32	&	 72.78 $\pm$ 0.51  \\	
SGC(Incep)-2	&	90.97 $\pm$ 0.17	&	94.47 $\pm$ 0.22	&	94.54 $\pm$ 0.08	&	83.74 $\pm$ 0.53	&	  72.47 $\pm$ 0.61 \\	
SIGN-2	&	90.89 $\pm$ 0.20	&	\textbf{94.59 $\pm$ 0.19}	&	94.09 $\pm$ 0.12	&	83.92 $\pm$ 0.43	&	  73.27 $\pm$ 0.53 \\	
	\hline										
Cov(X)	&	82.42 $\pm$ 0.14	&	89.45 $\pm$ 0.26	&	91.70 $\pm$ 0.15	&	69.24 $\pm$ 0.56	&	 66.79 $\pm$ 0.62	\\
ASE	&	77.61 $\pm$ 0.20	&	85.84 $\pm$ 0.22	&	75.24 $\pm$ 0.21	&	72.84 $\pm$ 0.48	&	51.73 $\pm$ 1.67	\\
A\_X	&	89.97 $\pm$ 0.21	&	94.22 $\pm$ 0.22	&	93.69 $\pm$ 0.13	&	80.89 $\pm$ 0.56	&	69.84 $\pm$ 0.71	\\
	\hline										
GCN$*$	&	90.49	&	93.91	&	93.32	&	85.77	&	 73.68 \\	
GAT$*$	&	-	&	-	&	-	&	 \textbf{86.37} 	&	 74.32\\	
Geom-GCN$*$	&	-	&	-	&	-	&	84.93	&	 \textbf{75.14}\\	
GCNII-10 	&	 90.75 $\pm$ 0.16 	&	 93.86 $\pm$ 0.18 	&	 94.32 $\pm$ 0.26	&	 84.14 $\pm$ 0.47 	&	 72.17 $\pm$ 0.66 \\	
GPR-GNN-10 	&	 87.62 $\pm$ 0.85 	&	 93.52 $\pm$ 0.39 	&	 \textbf{94.81 $\pm$ 0.27}	&	 85.77 $\pm$ 0.67 	&	 73.22 $\pm$ 0.73 \\	
JK-Concat-10 & 90.56 $\pm$ 0.18 & 93.84 $\pm$ 0.29 & 93.80 $\pm$ 0.12 & 84.41 $\pm$ 0.41 & 71.45 $\pm$ 0.56\\
\hline \hline
\end{tabular}
\end{table*}

\section{Discussion} \label{sec: discussion}
In this work, we propose a simple technique \texttt{PowerEmbed} that exploits the advantages of global spectral methods and local message-passing algorithms for graph representation learning. Specifically, \texttt{PowerEmbed} performs a layer-wise normalization that allows MPNNs to express the top-$k$ eigenvectors of a graph and capture global spectral information. We demonstrate the advantages of our techniques theoretically and empirically on both synthetic and real-world graphs. 

Our spectral-inspired GNN opens the door to many interesting future directions. Firstly, we can extend \texttt{PowerEmbed} to semi-supervised graph learning tasks allowing learnable weights in the message-passing layers: While \texttt{PowerEmbed} can still express the top-$k$ eigenvectors of the graph operators, it may converge to a function of eigenvectors driven by label supervision. We conjecture that the semi-supervised versions may be helpful in certain sparse graphs, where the graph eigenvectors are suboptimal in estimating community structure and the label signals can improve the inference performance (e.g., partially-labeled SBM \cite{cai2016inference, cencheng2022gee}). Furthermore, since \texttt{PowerEmbed} combines global spectral information with local MPNNs, it is natural to consider more powerful versions of \texttt{PowerEmbed} that are based on higher-order MPNNs. Yet it remains open to fully understand the relations between graph spatial information (i.e., symmetries) and graph spectral information (e.g., eigenvalues and eigenvectors) \cite{furer2010power, rattan2022WL}. We hope that our approach will inspire future research on principled combinations of global and local methods for graph representation learning.


\begin{ack}
We thank Rene Vidal for motivating this work, and Youngser Park for his valuable comments to this paper. We also thank the anonymous reviewers for giving us constructive feedback.
SV is partially supported by ONR N00014-22-1-2126, NSF CISE 2212457, an AI2AI Amazon research award, and the NSF–Simons Research Collaboration on the Mathematical and Scientific Foundations of Deep Learning (MoDL) (NSF DMS 2031985).
\end{ack}

\bibliographystyle{unsrtnat}
\bibliography{ref}

\begin{thebibliography}{80}
\providecommand{\natexlab}[1]{#1}
\providecommand{\url}[1]{\texttt{#1}}
\expandafter\ifx\csname urlstyle\endcsname\relax
  \providecommand{\doi}[1]{doi: #1}\else
  \providecommand{\doi}{doi: \begingroup \urlstyle{rm}\Url}\fi

\bibitem[Lei and Rinaldo(2015)]{Lei_2015}
Jing Lei and Alessandro Rinaldo.
\newblock Consistency of spectral clustering in stochastic block models.
\newblock \emph{The Annals of Statistics}, 43\penalty0 (1), feb 2015.
\newblock \doi{10.1214/14-aos1274}.
\newblock URL \url{https://doi.org/10.1214%2F14-aos1274}.

\bibitem[Athreya et~al.(2017)Athreya, Fishkind, Tang, Priebe, Park, Vogelstein,
  Levin, Lyzinski, and Qin]{athreya2017statistical}
Avanti Athreya, Donniell~E Fishkind, Minh Tang, Carey~E Priebe, Youngser Park,
  Joshua~T Vogelstein, Keith Levin, Vince Lyzinski, and Yichen Qin.
\newblock Statistical inference on random dot product graphs: a survey.
\newblock \emph{The Journal of Machine Learning Research}, 18\penalty0
  (1):\penalty0 8393--8484, 2017.

\bibitem[Abbe(2017)]{abbe2017community}
Emmanuel Abbe.
\newblock Community detection and stochastic block models: recent developments.
\newblock \emph{The Journal of Machine Learning Research}, 18\penalty0
  (1):\penalty0 6446--6531, 2017.

\bibitem[Alon and Yahav(2021)]{alon2021on}
Uri Alon and Eran Yahav.
\newblock On the bottleneck of graph neural networks and its practical
  implications.
\newblock In \emph{International Conference on Learning Representations}, 2021.
\newblock URL \url{https://openreview.net/forum?id=i80OPhOCVH2}.

\bibitem[{Li} et~al.(2018){Li}, {Han}, and {Wu}]{li2018deeper}
Q.~{Li}, Z.~{Han}, and X.-M. {Wu}.
\newblock {Deeper Insights into Graph Convolutional Networks for
  Semi-Supervised Learning}.
\newblock In \emph{The Thirty-Second AAAI Conference on Artificial
  Intelligence}. AAAI, 2018.

\bibitem[Oono and Suzuki(2019)]{oono2019graph}
Kenta Oono and Taiji Suzuki.
\newblock Graph neural networks exponentially lose expressive power for node
  classification.
\newblock In \emph{International Conference on Learning Representations}, 2019.

\bibitem[Topping et~al.(2021)Topping, Giovanni, Chamberlain, Dong, and
  Bronstein]{topping2021understanding}
Jake Topping, Francesco~Di Giovanni, Benjamin~Paul Chamberlain, Xiaowen Dong,
  and Michael~M. Bronstein.
\newblock Understanding over-squashing and bottlenecks on graphs via curvature,
  2021.

\bibitem[Pei et~al.(2020)Pei, Wei, Chang, Lei, and Yang]{Pei2020Geom-GCN}
Hongbin Pei, Bingzhe Wei, Kevin Chen-Chuan Chang, Yu~Lei, and Bo~Yang.
\newblock Geom-gcn: Geometric graph convolutional networks.
\newblock In \emph{International Conference on Learning Representations}, 2020.
\newblock URL \url{https://openreview.net/forum?id=S1e2agrFvS}.

\bibitem[Zhu et~al.(2020)Zhu, Yan, Zhao, Heimann, Akoglu, and
  Koutra]{zhu2020beyond}
Jiong Zhu, Yujun Yan, Lingxiao Zhao, Mark Heimann, Leman Akoglu, and Danai
  Koutra.
\newblock Beyond homophily in graph neural networks: Current limitations and
  effective designs.
\newblock \emph{Advances in Neural Information Processing Systems},
  33:\penalty0 7793--7804, 2020.

\bibitem[Ma et~al.(2022)Ma, Liu, Shah, and Tang]{ma2022is}
Yao Ma, Xiaorui Liu, Neil Shah, and Jiliang Tang.
\newblock Is homophily a necessity for graph neural networks?
\newblock In \emph{International Conference on Learning Representations}, 2022.
\newblock URL \url{https://openreview.net/forum?id=ucASPPD9GKN}.

\bibitem[Li et~al.(2019)Li, Muller, Thabet, and Ghanem]{li2019deepgcns}
Guohao Li, Matthias Muller, Ali Thabet, and Bernard Ghanem.
\newblock Deepgcns: Can gcns go as deep as cnns?
\newblock In \emph{Proceedings of the IEEE/CVF international conference on
  computer vision}, pages 9267--9276, 2019.

\bibitem[Cai et~al.(2021)Cai, Luo, Xu, He, Liu, and Wang]{cai2021graphnorm}
Tianle Cai, Shengjie Luo, Keyulu Xu, Di~He, Tie-yan Liu, and Liwei Wang.
\newblock Graphnorm: A principled approach to accelerating graph neural network
  training.
\newblock In \emph{International Conference on Machine Learning}, pages
  1204--1215. PMLR, 2021.

\bibitem[Veličković et~al.(2018)Veličković, Cucurull, Casanova, Romero,
  Liò, and Bengio]{velic2018graph}
Petar Veličković, Guillem Cucurull, Arantxa Casanova, Adriana Romero, Pietro
  Liò, and Yoshua Bengio.
\newblock Graph attention networks, 2018.

\bibitem[Xu et~al.(2018{\natexlab{a}})Xu, Li, Tian, Sonobe, Kawarabayashi, and
  Jegelka]{xu2018jump}
Keyulu Xu, Chengtao Li, Yonglong Tian, Tomohiro Sonobe, Ken-ichi Kawarabayashi,
  and Stefanie Jegelka.
\newblock Representation learning on graphs with jumping knowledge networks.
\newblock In Jennifer Dy and Andreas Krause, editors, \emph{Proceedings of the
  35th International Conference on Machine Learning}, volume~80 of
  \emph{Proceedings of Machine Learning Research}, pages 5453--5462. PMLR,
  10--15 Jul 2018{\natexlab{a}}.
\newblock URL \url{https://proceedings.mlr.press/v80/xu18c.html}.

\bibitem[Ying et~al.(2021)Ying, Cai, Luo, Zheng, Ke, He, Shen, and
  Liu]{ying2021transformers}
Chengxuan Ying, Tianle Cai, Shengjie Luo, Shuxin Zheng, Guolin Ke, Di~He,
  Yanming Shen, and Tie-Yan Liu.
\newblock Do transformers really perform bad for graph representation?
\newblock \emph{arXiv preprint arXiv:2106.05234}, 2021.

\bibitem[Kreuzer et~al.(2021)Kreuzer, Beaini, Hamilton, L{\'e}tourneau, and
  Tossou]{kreuzer2021rethinking}
Devin Kreuzer, Dominique Beaini, Will Hamilton, Vincent L{\'e}tourneau, and
  Prudencio Tossou.
\newblock Rethinking graph transformers with spectral attention.
\newblock \emph{Advances in Neural Information Processing Systems}, 34, 2021.

\bibitem[Huang and Villar(2021)]{HuangWL2021}
Ningyuan Huang and Soledad Villar.
\newblock A short tutorial on the weisfeiler-lehman test and its variants.
\newblock In \emph{ICASSP 2021 - 2021 IEEE International Conference on
  Acoustics, Speech and Signal Processing (ICASSP)}, pages 8533--8537, 2021.
\newblock \doi{10.1109/ICASSP39728.2021.9413523}.

\bibitem[Sato(2020)]{sato2020survey}
Ryoma Sato.
\newblock A survey on the expressive power of graph neural networks.
\newblock \emph{arXiv preprint arXiv:2003.04078}, 2020.

\bibitem[Xu et~al.(2018{\natexlab{b}})Xu, Hu, Leskovec, and
  Jegelka]{xu2018powerful}
Keyulu Xu, Weihua Hu, Jure Leskovec, and Stefanie Jegelka.
\newblock How powerful are graph neural networks?
\newblock \emph{arXiv preprint arXiv:1810.00826}, 2018{\natexlab{b}}.

\bibitem[Morris et~al.(2019)Morris, Ritzert, Fey, Hamilton, Lenssen, Rattan,
  and Grohe]{morris2019weisfeiler}
Christopher Morris, Martin Ritzert, Matthias Fey, William~L Hamilton, Jan~Eric
  Lenssen, Gaurav Rattan, and Martin Grohe.
\newblock Weisfeiler and leman go neural: Higher-order graph neural networks.
\newblock In \emph{Proceedings of the AAAI conference on artificial
  intelligence}, volume~33, pages 4602--4609, 2019.

\bibitem[Morris et~al.(2020)Morris, Rattan, and Mutzel]{morris2020weisfeiler}
Christopher Morris, Gaurav Rattan, and Petra Mutzel.
\newblock Weisfeiler and leman go sparse: Towards scalable higher-order graph
  embeddings.
\newblock \emph{Advances in Neural Information Processing Systems},
  33:\penalty0 21824--21840, 2020.

\bibitem[Bodnar et~al.(2021)Bodnar, Frasca, Wang, Otter, Montufar, Lio, and
  Bronstein]{bodnar2021weisfeiler}
Cristian Bodnar, Fabrizio Frasca, Yuguang Wang, Nina Otter, Guido~F Montufar,
  Pietro Lio, and Michael Bronstein.
\newblock Weisfeiler and lehman go topological: Message passing simplicial
  networks.
\newblock In \emph{International Conference on Machine Learning}, pages
  1026--1037. PMLR, 2021.

\bibitem[Cotta et~al.(2021)Cotta, Morris, and Ribeiro]{cotta2021reconstruction}
Leonardo Cotta, Christopher Morris, and Bruno Ribeiro.
\newblock Reconstruction for powerful graph representations.
\newblock \emph{Advances in Neural Information Processing Systems}, 34, 2021.

\bibitem[Chen et~al.(2019)Chen, Villar, Chen, and Bruna]{chen2019equivalence}
Zhengdao Chen, Soledad Villar, Lei Chen, and Joan Bruna.
\newblock On the equivalence between graph isomorphism testing and function
  approximation with gnns.
\newblock In \emph{Advances in Neural Information Processing Systems}, pages
  15868--15876, 2019.

\bibitem[Chen et~al.(2020{\natexlab{a}})Chen, Chen, Villar, and
  Bruna]{chen2020can}
Zhengdao Chen, Lei Chen, Soledad Villar, and Joan Bruna.
\newblock Can graph neural networks count substructures?
\newblock \emph{arXiv preprint arXiv:2002.04025}, 2020{\natexlab{a}}.

\bibitem[Morris et~al.(2021)Morris, Lipman, Maron, Rieck, Kriege, Grohe, Fey,
  and Borgwardt]{chris2022wl}
Christopher Morris, Yaron Lipman, Haggai Maron, Bastian Rieck, Nils~M. Kriege,
  Martin Grohe, Matthias Fey, and Karsten Borgwardt.
\newblock Weisfeiler and leman go machine learning: The story so far, 2021.
\newblock URL \url{https://arxiv.org/abs/2112.09992}.

\bibitem[Bruna et~al.(2013)Bruna, Zaremba, Szlam, and LeCun]{bruna2013spectral}
Joan Bruna, Wojciech Zaremba, Arthur Szlam, and Yann LeCun.
\newblock Spectral networks and locally connected networks on graphs.
\newblock \emph{arXiv preprint arXiv:1312.6203}, 2013.

\bibitem[Defferrard et~al.(2016)Defferrard, Bresson, and
  Vandergheynst]{defferrard2016convolutional}
Micha{\"e}l Defferrard, Xavier Bresson, and Pierre Vandergheynst.
\newblock Convolutional neural networks on graphs with fast localized spectral
  filtering.
\newblock \emph{Advances in neural information processing systems}, 29, 2016.

\bibitem[Perlmutter et~al.(2019)Perlmutter, Gao, Wolf, and
  Hirn]{perlmutter2019understanding}
Michael Perlmutter, Feng Gao, Guy Wolf, and Matthew Hirn.
\newblock Understanding graph neural networks with asymmetric geometric
  scattering transforms.
\newblock \emph{arXiv preprint arXiv:1911.06253}, 2019.

\bibitem[Liao et~al.(2019)Liao, Zhao, Urtasun, and Zemel]{liao2019lanczosnet}
Renjie Liao, Zhizhen Zhao, Raquel Urtasun, and Richard~S Zemel.
\newblock Lanczosnet: Multi-scale deep graph convolutional networks.
\newblock 2019.

\bibitem[Ruiz et~al.(2021)Ruiz, Gama, and Ribeiro]{ruiz2021}
Luana Ruiz, Fernando Gama, and Alejandro Ribeiro.
\newblock Graph neural networks: Architectures, stability, and transferability.
\newblock \emph{Proceedings of the IEEE}, 109\penalty0 (5):\penalty0 660--682,
  2021.
\newblock \doi{10.1109/JPROC.2021.3055400}.

\bibitem[Zhao and Akoglu(2019)]{pairnorm}
Lingxiao Zhao and Leman Akoglu.
\newblock Pairnorm: Tackling oversmoothing in gnns.
\newblock \emph{CoRR}, abs/1909.12223, 2019.

\bibitem[Yang et~al.(2020)Yang, Wang, Yao, Liu, and
  Abdelzaher]{yang2020revisiting}
Chaoqi Yang, Ruijie Wang, Shuochao Yao, Shengzhong Liu, and Tarek Abdelzaher.
\newblock Revisiting over-smoothing in deep gcns, 2020.

\bibitem[Zhou et~al.(2020)Zhou, Huang, Li, Zha, Chen, and Hu]{zhou2020towards}
Kaixiong Zhou, Xiao Huang, Yuening Li, Daochen Zha, Rui Chen, and Xia Hu.
\newblock Towards deeper graph neural networks with differentiable group
  normalization.
\newblock \emph{Advances in Neural Information Processing Systems},
  33:\penalty0 4917--4928, 2020.

\bibitem[Chen et~al.(2022)Chen, Tang, Qi, Li, and Xiao]{chen2022learning}
Yihao Chen, Xin Tang, Xianbiao Qi, Chun-Guang Li, and Rong Xiao.
\newblock Learning graph normalization for graph neural networks.
\newblock \emph{Neurocomputing}, 2022.

\bibitem[Dwivedi et~al.(2021)Dwivedi, Luu, Laurent, Bengio, and
  Bresson]{Dwivedi2022learnablePE}
Vijay~Prakash Dwivedi, Anh~Tuan Luu, Thomas Laurent, Yoshua Bengio, and Xavier
  Bresson.
\newblock Graph neural networks with learnable structural and positional
  representations.
\newblock 2021.
\newblock \doi{10.48550/ARXIV.2110.07875}.
\newblock URL \url{https://arxiv.org/abs/2110.07875}.

\bibitem[You et~al.(2019)You, Ying, and Leskovec]{You2019anchor}
Jiaxuan You, Rex Ying, and Jure Leskovec.
\newblock Position-aware graph neural networks.
\newblock In Kamalika Chaudhuri and Ruslan Salakhutdinov, editors,
  \emph{Proceedings of the 36th International Conference on Machine Learning},
  volume~97 of \emph{Proceedings of Machine Learning Research}, pages
  7134--7143. PMLR, 09--15 Jun 2019.
\newblock URL \url{https://proceedings.mlr.press/v97/you19b.html}.

\bibitem[Wenkel et~al.(2022)Wenkel, Min, Hirn, Perlmutter, and
  Wolf]{wenkel2022overcoming}
Frederik Wenkel, Yimeng Min, Matthew Hirn, Michael Perlmutter, and Guy Wolf.
\newblock Overcoming oversmoothness in graph convolutional networks via hybrid
  scattering networks.
\newblock \emph{arXiv preprint arXiv:2201.08932}, 2022.

\bibitem[Chen et~al.(2020{\natexlab{b}})Chen, Wei, Huang, Ding, and
  Li]{chen-plmr}
Ming Chen, Zhewei Wei, Zengfeng Huang, Bolin Ding, and Yaliang Li.
\newblock Simple and deep graph convolutional networks.
\newblock In Hal~Daumé III and Aarti Singh, editors, \emph{Proceedings of the
  37th International Conference on Machine Learning}, volume 119 of
  \emph{Proceedings of Machine Learning Research}, pages 1725--1735. PMLR,
  13--18 Jul 2020{\natexlab{b}}.
\newblock URL \url{http://proceedings.mlr.press/v119/chen20v.html}.

\bibitem[Liu et~al.(2018)Liu, Chen, Li, Zhou, Li, and Song]{geniepath}
Ziqi Liu, Chaochao Chen, Longfei Li, Jun Zhou, Xiaolong Li, and Le~Song.
\newblock Geniepath: Graph neural networks with adaptive receptive paths.
\newblock \emph{CoRR}, abs/1802.00910, 2018.

\bibitem[Liu et~al.(2021)Liu, Wang, and Ji]{liu2021non}
Meng Liu, Zhengyang Wang, and Shuiwang Ji.
\newblock Non-local graph neural networks.
\newblock \emph{IEEE Transactions on Pattern Analysis and Machine
  Intelligence}, 2021.

\bibitem[Xie et~al.(2020)Xie, Zhou, Liu, and Huang]{xie2020reinceptione}
Zhiwen Xie, Guangyou Zhou, Jin Liu, and Xiangji Huang.
\newblock Reinceptione: relation-aware inception network with joint
  local-global structural information for knowledge graph embedding.
\newblock In \emph{Proceedings of the 58th Annual Meeting of the Association
  for Computational Linguistics}, pages 5929--5939, 2020.

\bibitem[Wu et~al.(2021)Wu, Jain, Wright, Mirhoseini, Gonzalez, and
  Stoica]{wu2021representing}
Zhanghao Wu, Paras Jain, Matthew Wright, Azalia Mirhoseini, Joseph~E Gonzalez,
  and Ion Stoica.
\newblock Representing long-range context for graph neural networks with global
  attention.
\newblock \emph{Advances in Neural Information Processing Systems},
  34:\penalty0 13266--13279, 2021.

\bibitem[Rong et~al.(2020)Rong, Huang, Xu, and Huang]{rong2020dropedge}
Yu~Rong, Wenbing Huang, Tingyang Xu, and Junzhou Huang.
\newblock Dropedge: Towards deep graph convolutional networks on node
  classification, 2020.

\bibitem[Li et~al.(2020)Li, Zhang, Tian, Jin, Fardad, and Zafarani]{li2020sgcn}
Jiayu Li, Tianyun Zhang, Hao Tian, Shengmin Jin, Makan Fardad, and Reza
  Zafarani.
\newblock Sgcn: A graph sparsifier based on graph convolutional networks.
\newblock In \emph{Pacific-Asia Conference on Knowledge Discovery and Data
  Mining}, pages 275--287. Springer, 2020.

\bibitem[Chen et~al.(2021)Chen, Sui, Chen, Zhang, and Wang]{chen2021unified}
Tianlong Chen, Yongduo Sui, Xuxi Chen, Aston Zhang, and Zhangyang Wang.
\newblock A unified lottery ticket hypothesis for graph neural networks.
\newblock In \emph{International Conference on Machine Learning}, pages
  1695--1706. PMLR, 2021.

\bibitem[Zeng et~al.(2021{\natexlab{a}})Zeng, Zhang, Xia, Srivastava, Malevich,
  Kannan, Prasanna, Jin, and Chen]{zeng2021deep}
Hanqing Zeng, Muhan Zhang, Yinglong Xia, Ajitesh Srivastava, Andrey Malevich,
  Rajgopal Kannan, Viktor Prasanna, Long Jin, and Ren Chen.
\newblock Deep graph neural networks with shallow subgraph samplers,
  2021{\natexlab{a}}.

\bibitem[Yoon et~al.(2021)Yoon, Gervet, Shi, Niu, He, and
  Yang]{yoon2021performance}
Minji Yoon, Th{\'e}ophile Gervet, Baoxu Shi, Sufeng Niu, Qi~He, and Jaewon
  Yang.
\newblock Performance-adaptive sampling strategy towards fast and accurate
  graph neural networks.
\newblock In \emph{Proceedings of the 27th ACM SIGKDD Conference on Knowledge
  Discovery \& Data Mining}, pages 2046--2056, 2021.

\bibitem[Zeng et~al.(2021{\natexlab{b}})Zeng, Zhang, Xia, Srivastava, Malevich,
  Kannan, Prasanna, Jin, and Chen]{zeng2021decoupling}
Hanqing Zeng, Muhan Zhang, Yinglong Xia, Ajitesh Srivastava, Andrey Malevich,
  Rajgopal Kannan, Viktor Prasanna, Long Jin, and Ren Chen.
\newblock Decoupling the depth and scope of graph neural networks.
\newblock \emph{Advances in Neural Information Processing Systems}, 34,
  2021{\natexlab{b}}.

\bibitem[Gama et~al.(2018)Gama, Ribeiro, and Bruna]{gama2018diffusion}
Fernando Gama, Alejandro Ribeiro, and Joan Bruna.
\newblock Diffusion scattering transforms on graphs.
\newblock \emph{arXiv preprint arXiv:1806.08829}, 2018.

\bibitem[Mousavi et~al.(2017)Mousavi, Safayani, Mirzaei, and
  Bahonar]{mousavi2017hierarchical}
Seyedeh~Fatemeh Mousavi, Mehran Safayani, Abdolreza Mirzaei, and Hoda Bahonar.
\newblock Hierarchical graph embedding in vector space by graph pyramid.
\newblock \emph{Pattern Recognition}, 61:\penalty0 245--254, 2017.

\bibitem[Ying et~al.(2018)Ying, You, Morris, Ren, Hamilton, and
  Leskovec]{ying2018hierarchical}
Zhitao Ying, Jiaxuan You, Christopher Morris, Xiang Ren, Will Hamilton, and
  Jure Leskovec.
\newblock Hierarchical graph representation learning with differentiable
  pooling.
\newblock \emph{Advances in neural information processing systems}, 31, 2018.

\bibitem[Chen et~al.(2017)Chen, Li, and Bruna]{chen2017supervised}
Zhengdao Chen, Xiang Li, and Joan Bruna.
\newblock Supervised community detection with line graph neural networks.
\newblock \emph{arXiv preprint arXiv:1705.08415}, 2017.

\bibitem[Jin et~al.(2019)Jin, Liu, Li, He, and Zhang]{jin2019graph}
Di~Jin, Ziyang Liu, Weihao Li, Dongxiao He, and Weixiong Zhang.
\newblock Graph convolutional networks meet markov random fields:
  Semi-supervised community detection in attribute networks.
\newblock In \emph{Proceedings of the AAAI conference on artificial
  intelligence}, volume~33, pages 152--159, 2019.

\bibitem[Bianchi et~al.(2020)Bianchi, Grattarola, and
  Alippi]{bianchi20graphpool}
Filippo~Maria Bianchi, Daniele Grattarola, and Cesare Alippi.
\newblock Spectral clustering with graph neural networks for graph pooling.
\newblock In Hal~Daumé III and Aarti Singh, editors, \emph{Proceedings of the
  37th International Conference on Machine Learning}, volume 119 of
  \emph{Proceedings of Machine Learning Research}, pages 874--883. PMLR, 13--18
  Jul 2020.
\newblock URL \url{https://proceedings.mlr.press/v119/bianchi20a.html}.

\bibitem[Lyzinski et~al.(2014)Lyzinski, Sussman, Tang, Athreya, and
  Priebe]{lyzinski2014perfect}
Vince Lyzinski, Daniel~L Sussman, Minh Tang, Avanti Athreya, and Carey~E
  Priebe.
\newblock Perfect clustering for stochastic blockmodel graphs via adjacency
  spectral embedding.
\newblock \emph{Electronic journal of statistics}, 8\penalty0 (2):\penalty0
  2905--2922, 2014.

\bibitem[Von~Luxburg et~al.(2008)Von~Luxburg, Belkin, and
  Bousquet]{von2008consistency}
Ulrike Von~Luxburg, Mikhail Belkin, and Olivier Bousquet.
\newblock Consistency of spectral clustering.
\newblock \emph{The Annals of Statistics}, pages 555--586, 2008.

\bibitem[Kipf and Welling(2016)]{kipf2016semi}
Thomas~N Kipf and Max Welling.
\newblock Semi-supervised classification with graph convolutional networks.
\newblock \emph{arXiv preprint arXiv:1609.02907}, 2016.

\bibitem[Wu et~al.(2019)Wu, Souza, Zhang, Fifty, Yu, and
  Weinberger]{wu2019simplifying}
Felix Wu, Amauri Souza, Tianyi Zhang, Christopher Fifty, Tao Yu, and Kilian
  Weinberger.
\newblock Simplifying graph convolutional networks.
\newblock In \emph{International Conference on Machine Learning}, pages
  6861--6871, 2019.

\bibitem[Cai and Wang(2020)]{cai2020note}
Chen Cai and Yusu Wang.
\newblock A note on over-smoothing for graph neural networks, 2020.

\bibitem[Hua et~al.(1999)Hua, Xiang, Chen, Abed-Meraim, and Miao]{hua1999new}
Yingbo Hua, Yong Xiang, Tianping Chen, Karim Abed-Meraim, and Yongfeng Miao.
\newblock A new look at the power method for fast subspace tracking.
\newblock \emph{Digital Signal Processing}, 9\penalty0 (4):\penalty0 297--314,
  1999.

\bibitem[Lim et~al.(2022)Lim, Robinson, Zhao, Smidt, Sra, Maron, and
  Jegelka]{lim2022sign}
Derek Lim, Joshua Robinson, Lingxiao Zhao, Tess Smidt, Suvrit Sra, Haggai
  Maron, and Stefanie Jegelka.
\newblock Sign and basis invariant networks for spectral graph representation
  learning.
\newblock \emph{arXiv preprint arXiv:2202.13013}, 2022.

\bibitem[Maron et~al.(2018)Maron, Ben-Hamu, Shamir, and
  Lipman]{maron2018invariant}
Haggai Maron, Heli Ben-Hamu, Nadav Shamir, and Yaron Lipman.
\newblock Invariant and equivariant graph networks.
\newblock \emph{arXiv preprint arXiv:1812.09902}, 2018.

\bibitem[Frasca et~al.(2020)Frasca, Rossi, Eynard, Chamberlain, Bronstein, and
  Monti]{frasca2020SIGN}
Fabrizio Frasca, Emanuele Rossi, Davide Eynard, Ben Chamberlain, Michael
  Bronstein, and Federico Monti.
\newblock Sign: Scalable inception graph neural networks, 2020.
\newblock URL \url{https://arxiv.org/abs/2004.11198}.

\bibitem[Veli{\v{c}}kovi{\'c} et~al.(2017)Veli{\v{c}}kovi{\'c}, Cucurull,
  Casanova, Romero, Lio, and Bengio]{velivckovic2017graph}
Petar Veli{\v{c}}kovi{\'c}, Guillem Cucurull, Arantxa Casanova, Adriana Romero,
  Pietro Lio, and Yoshua Bengio.
\newblock Graph attention networks.
\newblock \emph{arXiv preprint arXiv:1710.10903}, 2017.

\bibitem[Chien et~al.(2020)Chien, Peng, Li, and Milenkovic]{chien2020adaptive}
Eli Chien, Jianhao Peng, Pan Li, and Olgica Milenkovic.
\newblock Adaptive universal generalized pagerank graph neural network.
\newblock \emph{arXiv preprint arXiv:2006.07988}, 2020.

\bibitem[Chen et~al.(2020{\natexlab{c}})Chen, Wei, Huang, Ding, and
  Li]{chen2020simple}
Ming Chen, Zhewei Wei, Zengfeng Huang, Bolin Ding, and Yaliang Li.
\newblock Simple and deep graph convolutional networks.
\newblock In \emph{International Conference on Machine Learning}, pages
  1725--1735. PMLR, 2020{\natexlab{c}}.

\bibitem[Fey and Lenssen(2019)]{Fey2019pytorchgeo}
Matthias Fey and Jan~E. Lenssen.
\newblock Fast graph representation learning with {PyTorch Geometric}.
\newblock In \emph{ICLR Workshop on Representation Learning on Graphs and
  Manifolds}, 2019.

\bibitem[Luo et~al.(2022)Luo, Luo, Yan, and Chen]{luo2022inferring}
Yi~Luo, Guangchun Luo, Ke~Yan, and Aiguo Chen.
\newblock Inferring from references with differences for semi-supervised node
  classification on graphs.
\newblock \emph{Mathematics}, 10\penalty0 (8):\penalty0 1262, 2022.

\bibitem[Kingma and Ba(2014)]{kingma2014adam}
Diederik~P Kingma and Jimmy Ba.
\newblock Adam: A method for stochastic optimization.
\newblock \emph{arXiv preprint arXiv:1412.6980}, 2014.

\bibitem[Gasteiger et~al.(2019)Gasteiger, Wei{\ss}enberger, and
  G{\"u}nnemann]{gasteiger_diffusion_2019}
Johannes Gasteiger, Stefan Wei{\ss}enberger, and Stephan G{\"u}nnemann.
\newblock Diffusion improves graph learning.
\newblock In \emph{Conference on Neural Information Processing Systems
  (NeurIPS)}, 2019.

\bibitem[Cai et~al.(2016)Cai, Liang, and Rakhlin]{cai2016inference}
T~Tony Cai, Tengyuan Liang, and Alexander Rakhlin.
\newblock Inference via message passing on partially labeled stochastic block
  models.
\newblock \emph{arXiv preprint arXiv:1603.06923}, 2016.

\bibitem[Shen et~al.(2021)Shen, Wang, and Priebe]{cencheng2022gee}
Cencheng Shen, Qizhe Wang, and Carey~E. Priebe.
\newblock One-hot graph encoder embedding, 2021.
\newblock URL \url{https://arxiv.org/abs/2109.13098}.

\bibitem[F{\"u}rer(2010)]{furer2010power}
Martin F{\"u}rer.
\newblock On the power of combinatorial and spectral invariants.
\newblock \emph{Linear algebra and its applications}, 432\penalty0
  (9):\penalty0 2373--2380, 2010.

\bibitem[Rattan and Seppelt(2021)]{rattan2022WL}
Gaurav Rattan and Tim Seppelt.
\newblock Weisfeiler--leman and graph spectra, 2021.
\newblock URL \url{https://arxiv.org/abs/2103.02972}.

\bibitem[Von~Luxburg(2007)]{von2007tutorial}
Ulrike Von~Luxburg.
\newblock A tutorial on spectral clustering.
\newblock \emph{Statistics and computing}, 17\penalty0 (4):\penalty0 395--416,
  2007.

\bibitem[Rozemberczki et~al.(2021)Rozemberczki, Allen, and
  Sarkar]{rozemberczki2021multi}
Benedek Rozemberczki, Carl Allen, and Rik Sarkar.
\newblock Multi-scale attributed node embedding.
\newblock \emph{Journal of Complex Networks}, 9\penalty0 (2):\penalty0 cnab014,
  2021.

\bibitem[Shchur et~al.(2018)Shchur, Mumme, Bojchevski, and
  G{\"u}nnemann]{shchur2018pitfalls}
Oleksandr Shchur, Maximilian Mumme, Aleksandar Bojchevski, and Stephan
  G{\"u}nnemann.
\newblock Pitfalls of graph neural network evaluation.
\newblock \emph{arXiv preprint arXiv:1811.05868}, 2018.

\bibitem[Sen et~al.(2008)Sen, Namata, Bilgic, Getoor, Galligher, and
  Eliassi-Rad]{sen2008collective}
Prithviraj Sen, Galileo Namata, Mustafa Bilgic, Lise Getoor, Brian Galligher,
  and Tina Eliassi-Rad.
\newblock Collective classification in network data.
\newblock \emph{AI magazine}, 29\penalty0 (3):\penalty0 93--93, 2008.

\bibitem[Gu et~al.(2021)Gu, Tandon, Ahn, and Radicchi]{gu2021principled}
Weiwei Gu, Aditya Tandon, Yong-Yeol Ahn, and Filippo Radicchi.
\newblock Principled approach to the selection of the embedding dimension of
  networks.
\newblock \emph{Nature Communications}, 12\penalty0 (1):\penalty0 1--10, 2021.

\end{thebibliography}

\clearpage

\appendix

\section{Proofs of \texttt{PowerEmbed} properties} 
\label{app.powerembed}

\begin{proof}[Proof of Proposition \ref{prop.limit}]

Recall $S \in \R^{n \times n}$ and  $U(t) \in \R^{n \times k}$ for all $t \ge 0$. Since $S \in \{A, \bar{A}, A_{rw}\}$, $S$ admits a spectral decomposition. This is obvious for the symmetric matrices $A, \bar{A}$. For non-symmetric $A_{rw}$, if $\lambda$ is an eigenvalue of $\bar{A}$ with eigenvector $w$, then $\lambda$ is also an eigenvalue of $A_{rw}$ with eigenvector $\tilde{D}^{-0.5} w$ \cite[Proposition.3.3]{von2007tutorial}. 

Assume the initialization $U_0$ is full rank, and the $k$-th and $(k+1)$-th eigenvalue of $S$ are distinct. Following the argument from Appendix A in \cite{hua1999new} we observe
\begin{equation}
    U(t+1) = S U(t) \left( U(t)^\top S^2 U(t) \right)^{-1} \label{eqn:iter}
\end{equation}
where $U(t) \in \mathbb R^{n\times k}$. Consider the spectral decomposition of $S$:
\begin{equation}
    S = [V_1 \,\, V_2] \left[ \begin{matrix} \Lambda_1 & 0 \\
    0 & \Lambda_2
    \end{matrix} \right] [V_1 \,\, V_2]^\top, \label{eqn:eigendecomp}
\end{equation}
where $V_1 \in \mathbb R^{n \times k}$, $V_2 \in \mathbb R^{n \times (n-k)}$, and $\Lambda_1 =\operatorname{diag}(\lambda_1, \ldots, \lambda_{k})$, $\Lambda_2=\operatorname{diag}(\lambda_{k+1}, \ldots, \lambda_{n})$ and  

Let $\Phi(t) = V_1^\top U(t) \in \R^{k \times k}$  and $\Omega(t)= V_2^\top U(t) \in \R^{(n-k) \times k}$ for $t \ge 0$. Multiplying $[V_1 \, \, V_2]^{\top}$ on both sides of \eqref{eqn:iter} and using \eqref{eqn:eigendecomp}, we have
\begin{equation}
    \left[\begin{matrix}
    \Phi(t+1) \\
    \Omega(t+1)
    \end{matrix}
    \right]
    =  
    \left[ \begin{matrix}
    \Lambda_1 & 0 \\
    0 & \Lambda_2
    \end{matrix} \right] 
    \left[\begin{matrix}
    \Phi(t) \\
    \Omega(t)
    \end{matrix}
    \right]
    Z(t), \label{eqn:block}
\end{equation}
where $Z(t) = \{ \Phi(t)^\top \Lambda_1^2 \Phi(t) + \Omega(t)^\top \Lambda_2^2 \Omega(t) \}^{-1}$. 

Since $\Phi(0) = V_1^{\top} U(0)$ is full rank, there exists $L$ such that $\Omega(0) = L \Phi(0)$. A simple induction argument shows that 
\begin{equation}
    \Omega(t) = \Lambda_2^{t} L \Lambda_1^{-t}\Phi(t). \label{eqn:vanish}
\end{equation}

Using the assumption that $\lambda_{k} > \lambda_{k+1}$, \eqref{eqn:vanish} implies that $\lim_{t \to \infty} \Omega(t) \to 0$. Therefore $\lim_{t \to \infty} Z(t) = \{ \Phi(t)^\top \Lambda_1^2 \Phi(t) \}^{-1}$. Then \eqref{eqn:block} implies 
\begin{equation}
    \Phi(t+1) = \Lambda_1 \Phi(t) Z(t) = \Lambda_{1}^{-1} \left(\Phi(t)^{\top} \right)^{-1} .
\end{equation}
Thus, upon convergence at $t=L$, $\Phi(L) \Phi(L)^{\top} = \Lambda_1^{-1}$ which implies $U(L) =  V_1  \Lambda_1^{-0.5}  W $ where $W$ is a unitary matrix. This proves that Algorithm \ref{alg:power_gnn} returns the top-$k$ eigenvectors of $S$ up to orthogonal transformation in $O(k)$.
\end{proof}

\begin{proof}[Proof of Proposition \ref{prop.equivariance}]
Algorithm \ref{alg:power_gnn} is permutation equivariant. Namely, for all $\Pi\in S_n$ permutation matrix and all $A\in \mathbb R^{n\times n}$ and all $X \in \mathbb R^{n\times k}$ we have that 
\begin{equation}
    \texttt{PowerEmbed}(\Pi \,A \,\Pi^\top, \Pi\, X) =  \Pi \,\texttt{PowerEmbed}(A, X).
\end{equation}
It suffices to observe that 
$\Pi A \Pi^\top \Pi X = \Pi A \, X$ and  $(\Pi A \, X)^\top (\Pi A \, X) = (A\, X)^\top (A\, X)$, making $\tilde{U}(t+1) ^{\top} \tilde{U}(t+1)$ permutation invariant, and therefore
$[\mathcal{T}_u(\tilde{U}(t+1) ^{\top} \tilde{U}(t+1) )]^{-1}$ is also permutation invariant.
Therefore at each step $\tilde U(t+1) (\Pi\, A\,\Pi^\top, \Pi\,X) = \Pi \tilde U(t+1)(A, X)$, making \texttt{PowerEmbed} permutation equivariant. 

The inception networks $g_{\theta'} \circ \texttt{MLPs}$ are permutation equivariant because the same function is applied to all the nodes. It implements the equivariance as weight sharing (similar to message-passing methods). 
\end{proof}

\section{Experiment details}\label{app.exp_details}

\subsection{Baselines} \label{app.baselines}

We perform \texttt{PowerEmbed} with different graph operators: $A$ (indicated by ``Power''), $A_{rw}$ (indicated by ``Power(RW)'') and $\bar{A}$ (indicated by ``Power(Lap)'').

For semi-supervised MPNNs including GCNII, GPR-GNN, we use hyperparameters that are are comparable with \texttt{PowerEmbed}: epochs~$=100$, learning rate~$=0.01$, dropout rate~$=0.5$, weight decay~$=0$, number of layers~$=10$. Specifically for GCNII, we use $\alpha=0.1$.

We remark that JK-Concat \cite{xu2018jump} can be viewed as the semi-supervised version of SIGN \cite{frasca2020SIGN} that uses learnable weights and nonlinearity in the message-passing layers. Thus, JK-Concat could produce layer-wise embeddings more adaptive to the inference task by learning end-to-end. Nonetheless, similar to other semi-supervised MPNNs, JK-Concat requires storing the graph at training and inference time, whereas unsupervised methods like \texttt{PowerEmbed} and SIGN first embed the graph as Euclidean features, which allows fast training and inference with computational complexity independent of the graph topology \cite{frasca2020SIGN}. Due to its similarity to \texttt{PowerEmbed} and SIGN, we adopt the same input feature dimensionality reduction strategy (see details in Appendix \ref{app.choice}) and $10$-layer concatenation for better comparison.

\subsection{Datasets}\label{app.datasets}

To evaluate \texttt{PowerEmbed}, we consider 5 dense graphs and 5 sparse graphs with different sizes and homophily ratios (see Tables \ref{tab:experiments}, \ref{tab:experiments-sparse} for summary statistics). 

\textit{Webpage networks. } In these networks, nodes represent webpages and edges are hyperlinks between them. Squirrel and Chameleon are created based on specific topics from Wikipedia \cite{rozemberczki2021multi}, where node features are informative nouns in the Wikipedia pages and the node labels are created in \cite{Pei2020Geom-GCN} which reflect the average monthly traffic of the webpage.  Wisconsin, Texas and Cornell are introduced in \cite{Pei2020Geom-GCN}, where node features are bag-of-words representation of the webpages, and node labels are the webpage category. All webpage networks we considered exhibit heterophily structure, where node with different class labels tend to be connected.

\textit{Co-purchase networks. } Computers and Photo are Amazon co-purchase networks introduced in \cite{shchur2018pitfalls}, where nodes represent goods, edges represent tendency of being bought together, node features are bag-of-words vector of product reviews, and node labels are product category.

\textit{Citation networks. } Cora and Citeseer are standard citation network benchmarks \cite{sen2008collective} where nodes represent documents, edges represent citation links, node features are bag-of-words vector from the dictionary, and node labels are the document category.

\textit{Co-authorship network. } Coauthor(CS) is introduced in \cite{shchur2018pitfalls} where nodes represent researchers, edges represent coauthorship, node features are bag-of-words vector from the paper keyword dictionary, and node labels are the research fields.

\subsection{The Choice of the Embedding dimension $k$} \label{app.choice}

\texttt{PowerEmbed}, similar to spectral embedding methods, crucially depends on the choice of embedding dimension $k$: choosing $k$ too small introduces large bias while choosing $k$ too large increases variance. We refer the interested reader to \cite{athreya2017statistical, gu2021principled} for in-depth discussions. 

In our experiments on real graphs, we choose $k=10$ for small graphs including Wisconsin, Texas, and Cornell; $k=100$ for other larger graphs. The same embedding dimension is used across all unsupervised methods (i.e., \texttt{PowerEmbed}, unnormalized counterparts, spectral methods) and JK-Concat that take the top-$k$ eigenvectors of the feature covariance matrix as input. The embedding dimension is chosen based on balancing the bias and variance as well as reducing memory complexity. On the other hand, we follow the original experiment design in semi-supervised MPNNs including GCNII and GPR-GNN, which use original node features and set $k$ the same as the feature dimension. We remark that the optimal choice of $k$ depends on each dataset and each embedding method, which is beyond the scope of current work but an interesting future research direction.

\section{Ablation study}\label{app.ablation}

\subsection{Synthetic graphs}
\label{app.ablation_sbm}

We report comprehensive simulation results on 2B-SBM models. Graph Laplacian operator $\bar{A}$ is used in all message-passing methods including \texttt{PowerEmbed}, SGC, GCN, GCNII, GPR-GNN. All experiments are done using $n=500$ nodes and repeated for 30 random runs. Figure \ref{fig:sbm_ablation} supplements Figure \ref{fig:power-sbm-simulation} with other baselines, and Figure \ref{fig:sbm_vary_density_ablation} shows a more fine-grained picture as the graph density decreases. We make the following observations.

\textit{Performance in sparse graphs.} In sparse graphs (e.g., Fig \ref{fig:sbm_ablation} (b), Fig \ref{fig:sbm_vary_density_ablation}), spectral methods and the last iterate of \texttt{PowerEmbed} perform poorly whereas \texttt{PowerEmbed} using all intermediate representations (including initial node features) performs well. Shallow MPNNs achieve strong performance (e.g., `SGC-2'', ``GCN-2''), including heterophilous graphs (Fig \ref{fig:sbm_ablation} (B)): as analyzed in \cite{ma2022is}, when nodes from the same class exhibit similar associativity pattern with other classes (as in the case of SBM graphs), heterophilous graphs are not necessarily hard to learn for shallows MPNNs. 

\textit{Over-smoothing.} Fig \ref{fig:sbm_vary_density_ablation} shows that Standard MPNNs  (e.g., `SGC-10'', ``GCN-10'') suffer from over-smoothing: increasing the number of layers degrades the performance. Moreover, the denser the graph, the worse of the performance degradation (e.g., `SGC-5'', ``GCN-5' perform better in sparse graphs than dense graphs), a phenomenon discussed in \cite{oono2019graph, rong2020dropedge}. On the other hand, the last iterate of \texttt{PowerEmbed} (right-most columns in Fig \ref{fig:sbm_ablation}) prevents over-smoothing in dense graphs. MPNNs that are specifically designed to avoid over-smoothing (e.g., GCNII, GPR-GNN) perform more stable as the number of layers increases, but suffer from high variance issues. \texttt{PowerEmbed} also prevents over-smoothing while achieving smaller variance. This illustrates the benefits of expressing global \textit{spectral} information over long-range spatial information under the SBM models, which we also empirically observe in some real-world graphs (see Table \ref{tab:experiments}, Table \ref{tab:experiments-sparse}).

\begin{figure}[h!]
  \centering
 \includegraphics[width=0.9\textwidth]{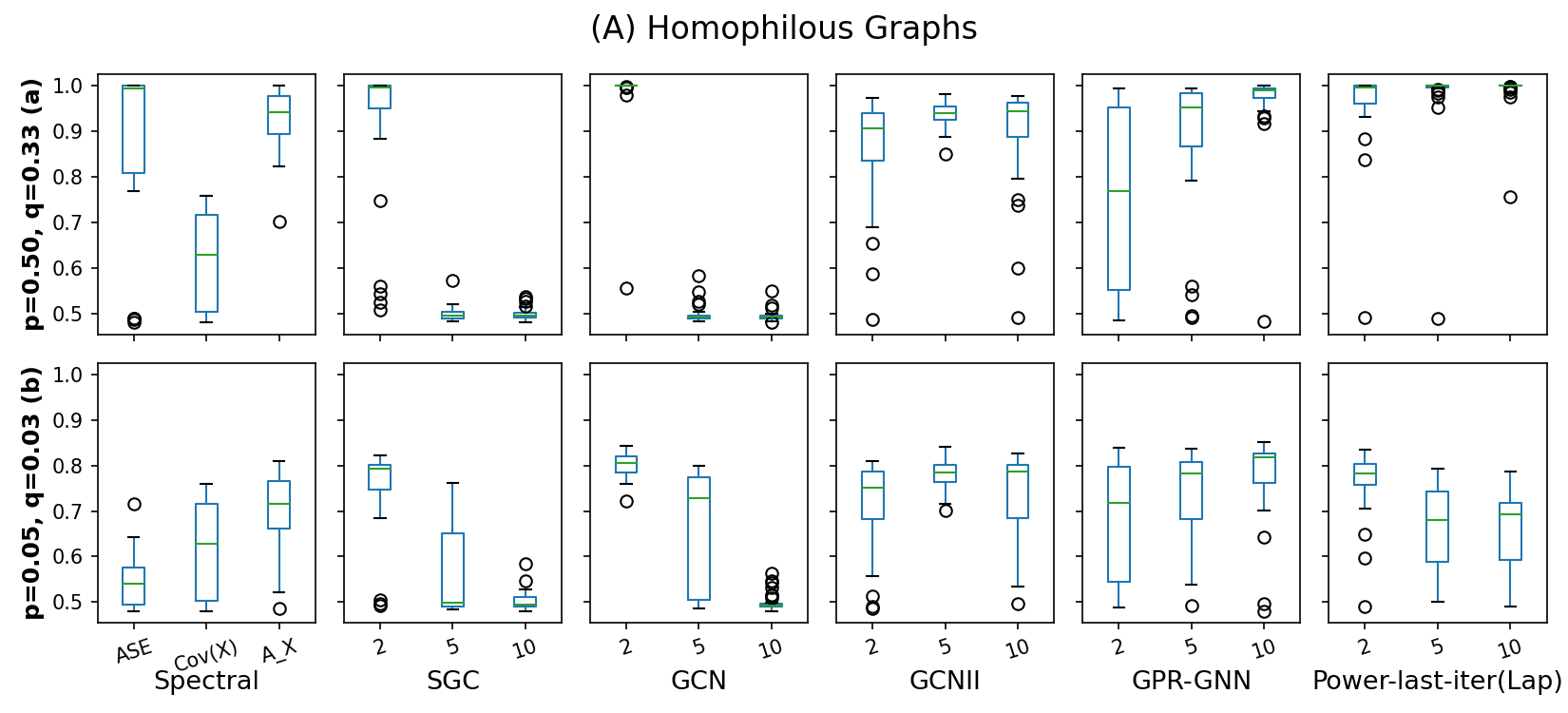}
 \includegraphics[width=0.9\textwidth]{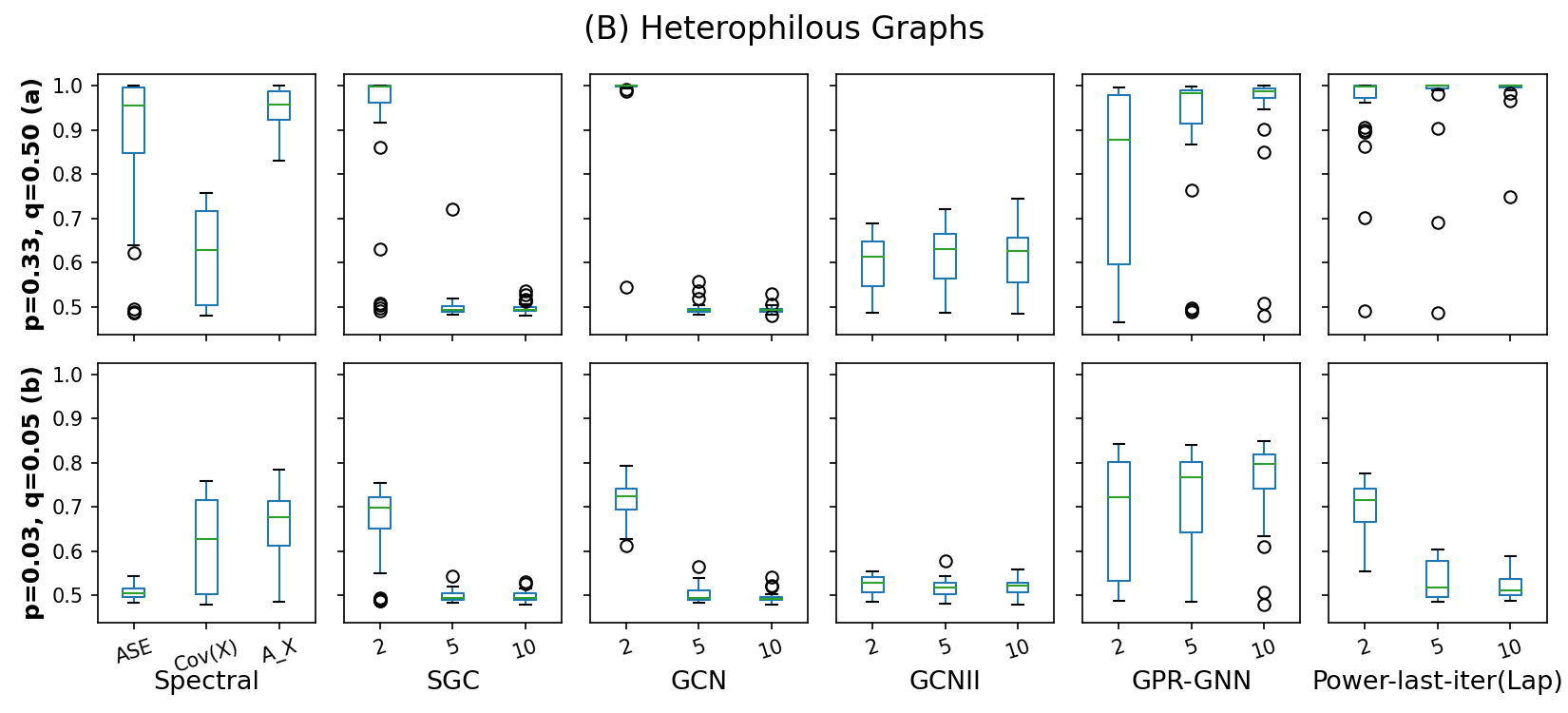}
  \caption{Supplementary to Figure \ref{fig:power-sbm-simulation} on other baselines: the last iterate of \texttt{PowerEmbed} perform well in dense graphs (both homophily and heterophily), but fail in sparse graphs, similar to spectral embedding (``ASE''). Standard MPNNs suffer from over-smoothing (``SGC-10'', ``GCN-10''). MPNNs encode long-range information can go deeper (``GCNII'', ``GPR-GNN''), albeit with higher variance.}
  \label{fig:sbm_ablation}
\end{figure}

\begin{figure}[h!]
  \centering
 \includegraphics[width=0.9\textwidth]{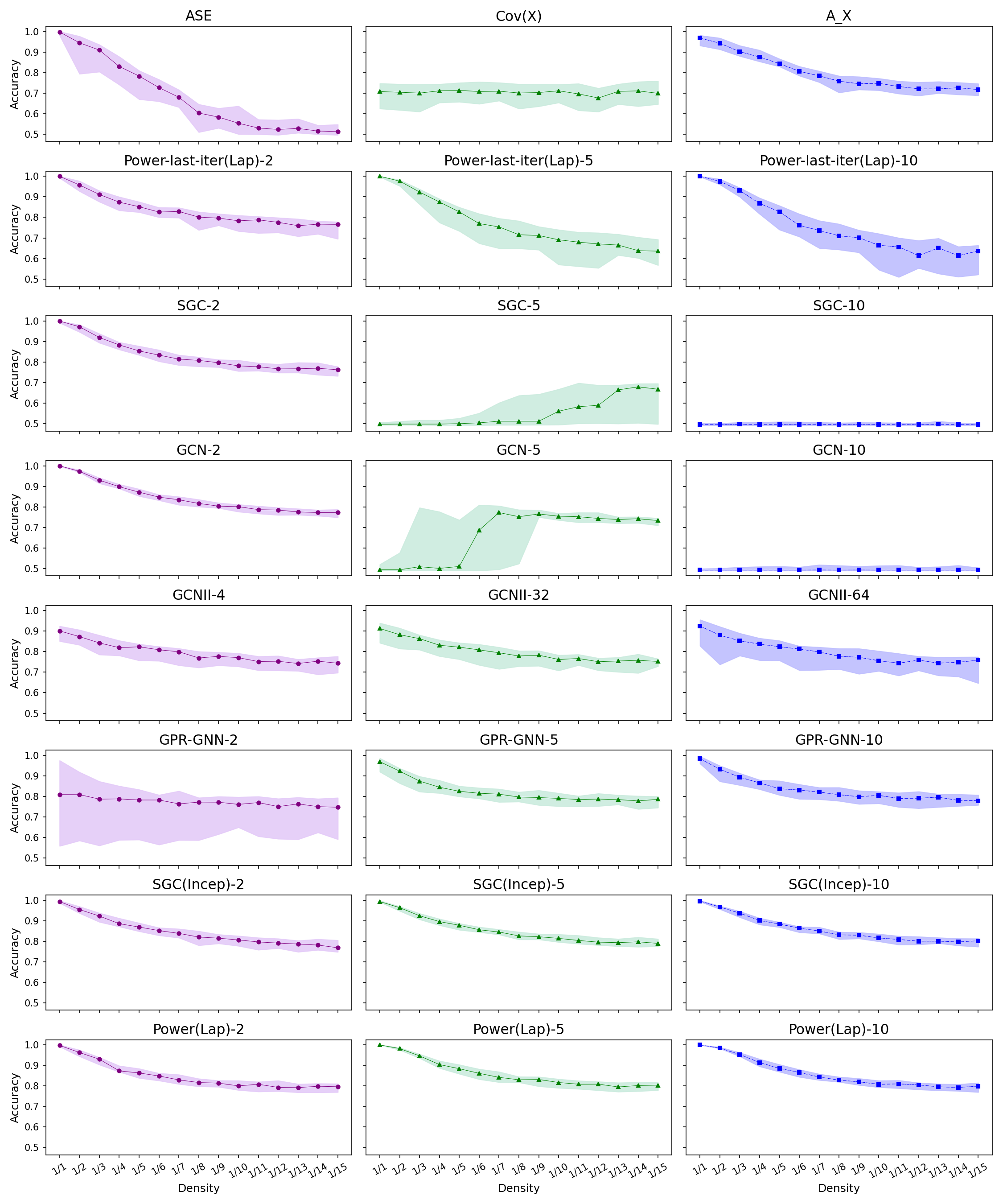}
  \caption{Performance w.r.t graph density changes in 2B-SBM model: $p=1/2 \times \text{X-axis},~q=1/3 \times \text{X-axis}$ where $p,q$ defines the block connection probability matrix $B$ in eqn~\eqref{eqn:2B-SBM}. ``Method-$k$'' denotes the corresponding method with $k$ number of layers. As density decreases from left to right, the performance of spectral embedding ``ASE'' and the last iterate of \texttt{PowerEmbed} (second row) degrade significantly, while shallow MPNNs degrade more gracefully (e.g.,``SGC-2'', ``GCN-2''). On the other hand, deep MPNNs (e.g.,``SGC-10'', ``GCN-10'') completely fail due to oversmoothing, while deep (last iterate of) \texttt{PowerEmbed} are more resilient. Finally, \texttt{PowerEmbed} and SGC(Incep) (bottom two rows) that use a list of intermediate embeddings perform consistently well in sparse graphs, robust to the choice of number of layers. } 
  \label{fig:sbm_vary_density_ablation}
\end{figure}

\clearpage
\subsection{Real-world graphs}
\label{app.ablation_real}
We compare the performance of \texttt{PowerEmbed} on real benchmarks using (1) all iterations; (2) last iteration; (3) the input feature and the last iteration. Figure \ref{fig:dense_ablation} shows that on dense graphs, even without layer-wise concatenation, \texttt{PowerEmbed} outperforms its unnormalized counterparts when the number of layers increases. This empirically demonstrates its ability to avoid over-smoothing.

\begin{figure}[h!]
  \centering
 \includegraphics[width=0.9\textwidth]{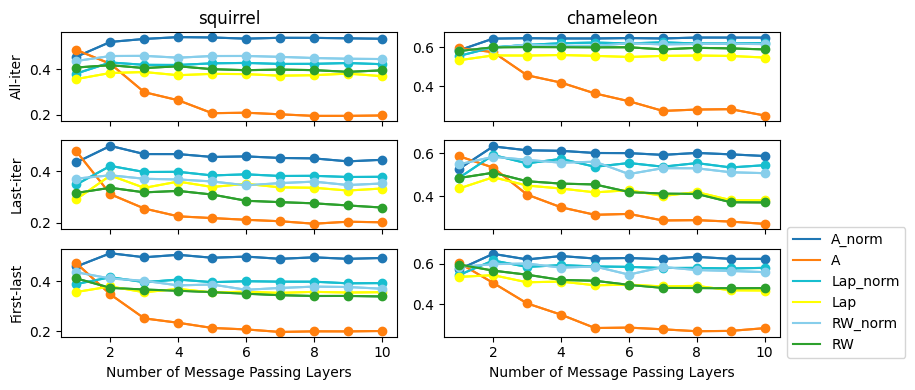}
 \includegraphics[width=0.9\textwidth]{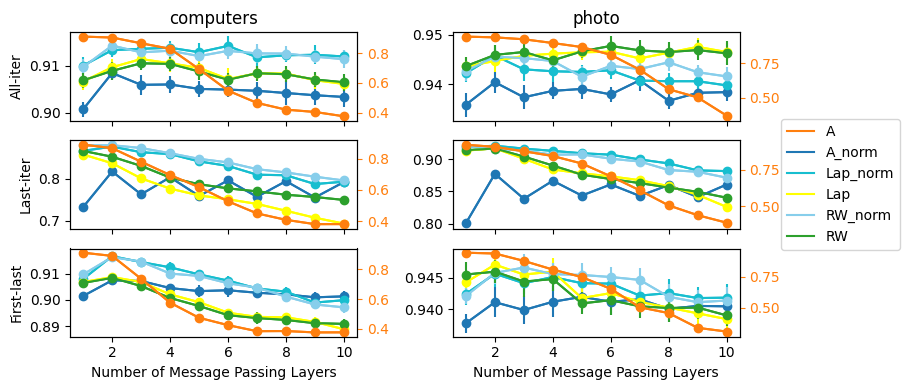} 
  \caption{In dense graphs, \texttt{PowerEmbed} (blue lines, annotated with ``\_norm'') outperforms its message-passing counterparts; We compare different graph operators $A$, $\bar{A}$ (``Lap''), and $A_{rw}$ (``RW''). We also study the effects of the number of iterations. Bottom row: The secondary Y-axis with orange labels correspond to performance of $A$.}
  \label{fig:dense_ablation}
\end{figure}

\end{document}